\documentclass{article}

\usepackage[utf8]{inputenc} 
\usepackage[T1]{fontenc}    
\usepackage{hyperref}       
\usepackage{url}            
\usepackage{booktabs}       
\usepackage{amsfonts}       
\usepackage{nicefrac}       
\usepackage{microtype}      

\usepackage[table,svgnames]{xcolor}       
\usepackage{graphicx}
\usepackage{wrapfig}
\usepackage{subcaption}
\usepackage{caption}
\usepackage{hyphenat}
\usepackage{setspace}
\usepackage{amsmath}
\usepackage{amssymb}
\usepackage{mathtools}
\usepackage{amsthm}
\usepackage{longtable}
\usepackage{multirow}
\usepackage{float}
\usepackage{tabularx}
\usepackage{enumitem}
\usepackage{threeparttable}
\usepackage{pifont}
\usepackage{algorithm}
\usepackage{algpseudocodex}
\usepackage{titletoc}
\usepackage{colortbl}
\usepackage{adjustbox}
\usepackage{stfloats}

\newcommand{\modelicon}[1]{%
    \raisebox{-0.15em}{%
        \includegraphics[height=1em]{#1}%
    }%
}

\newtheorem{lemma}{Lemma}
\newtheorem{theorem}{Theorem}

\usepackage{tikz}
\usetikzlibrary{shapes.geometric, arrows.meta, positioning, fit, calc, backgrounds, decorations.pathreplacing}

\usepackage[preprint]{my_style}

\usepackage{tcolorbox}
\tcbuselibrary{skins, breakable, listings}

\definecolor{brandblue}{RGB}{57,95,207}
\definecolor{linkblue}{HTML}{0064E0}
\definecolor{textgray}{HTML}{1C2B33}
\definecolor{boxbg}{HTML}{F1F4F7}
\definecolor{grey}{RGB}{128,128,128}
\hypersetup{
  colorlinks=true,
  linkcolor=brandblue,
  citecolor=brandblue,
  urlcolor=linkblue
}

\newtheorem{definition}{Definition}
\newtheorem{proposition}{Proposition}

\newcommand{\paperTitle}{ASSEMCAD: Production-Ready CAD \\ Assembly Generation from Natural Language}

\newcommand{\paperAuthors}{%
  {\sffamily\bfseries Yurui Dong$^{1,2*}$},
  {\sffamily\bfseries Shu Zou$^{1,4,*}$},
  {\sffamily\bfseries Siqi Li$^{1,5}$},
  {\sffamily\bfseries Nianchen Deng$^{1}$},
  {\sffamily\bfseries Hongbin Zhou$^{1}$},
  {\sffamily\bfseries Xuemeng Yang$^{1}$},
  {\sffamily\bfseries Pinlong Cai$^{1}$},
  {\sffamily\bfseries Licheng Wen$^{1,3}$},
  {\sffamily\bfseries Xinyu Cai$^{1\dag}$},
  {\sffamily\bfseries Botian Shi$^{1,3}$}
}

\newcommand{\paperAffiliations}{%
  {\normalsize $^1$ Shanghai Artificial Intelligence Laboratory},
  {\normalsize $^2$ Fudan University},
  {\normalsize $^3$ Shanghai Innovation Institute}\\
  {\normalsize $^4$ The Australian National University},
  {\normalsize $^5$ Zhejiang University}\\

}

\newcommand{\paperNotes}{%
  {\small $^*$ Equal Contribution},
  {\small $^\dag$ Corresponding Author}%
}

\newcommand{\publishDate}{\today}

\newcommand{\renderFrontBox}{%
    \tcbset{
    enhanced, frame hidden,
    colback=boxbg,
    left=0.5cm, right=0.5cm, top=0.5cm, bottom=0.5cm,
    arc=16pt,
    before skip=0pt,
    grow to left by=1.5pt, grow to right by=1.5pt,
    overlay={
    \node[anchor=north east, at=(frame.north east), xshift=-2.3cm, yshift=-0.5cm] 
        {\includegraphics[height=1cm]{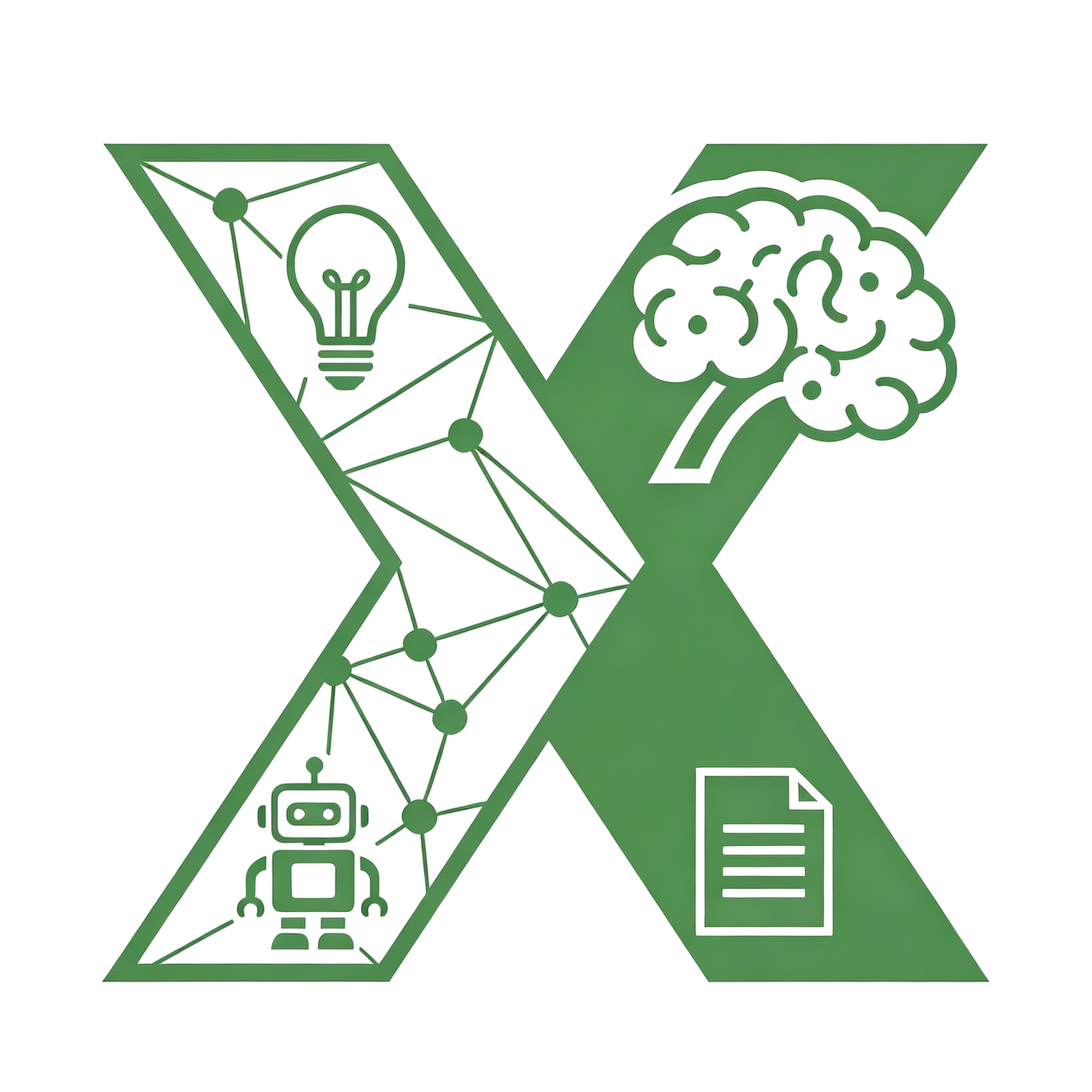}};
    \node[anchor=north east, at=(frame.north east), xshift=-0.5cm, yshift=-0.5cm] 
        {\includegraphics[height=1cm]{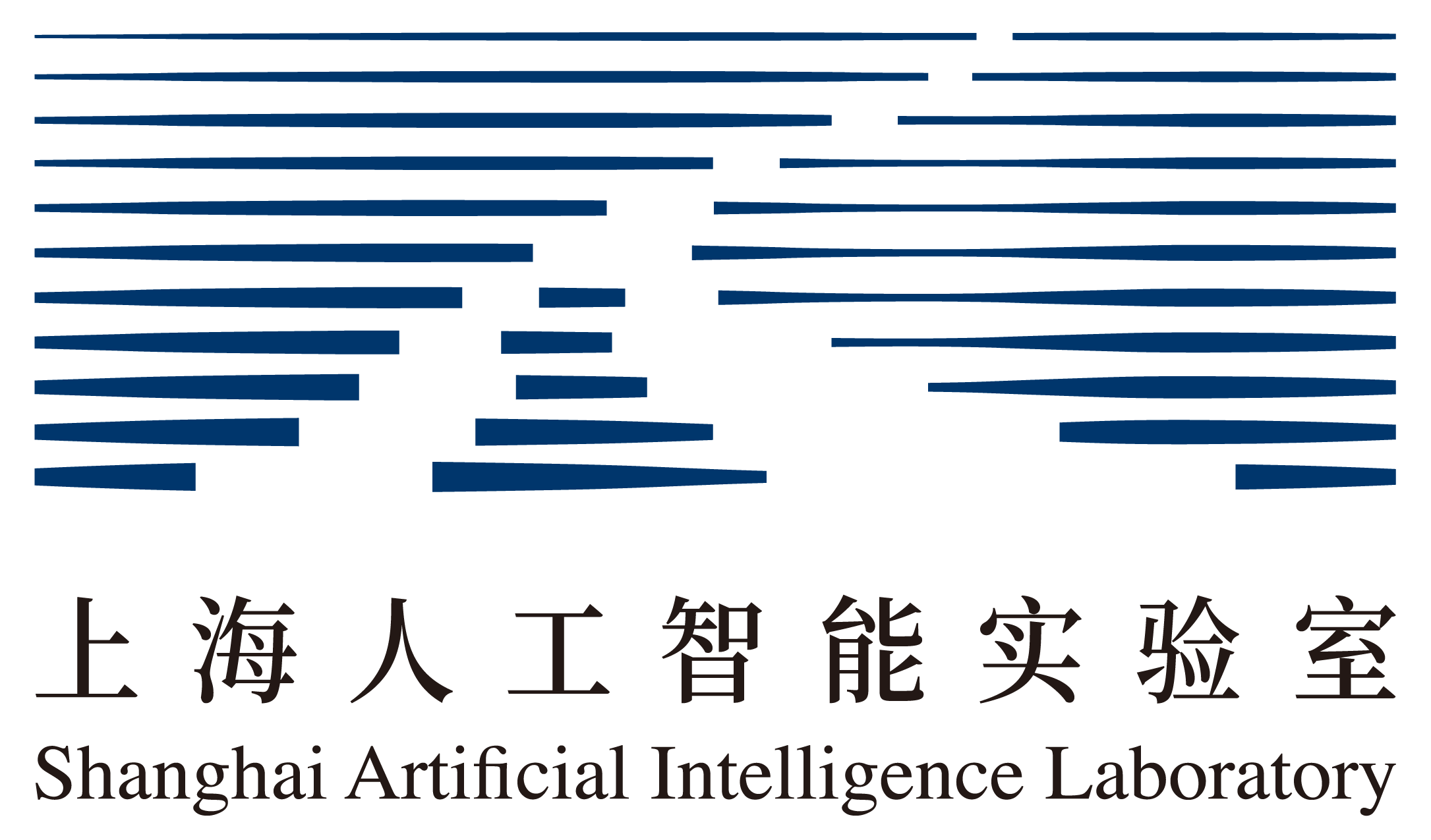}};
    }
  }%
  \begin{tcolorbox}
    \setlength{\parindent}{0cm}
    \setlength{\parskip}{0.5cm}
    {
      \setlength{\parskip}{0cm}
      \raggedright
      \nohyphens
      {
        \vskip 1cm 
        \setstretch{1.4}
        {\huge\sffamily\bfseries\textcolor{black}{\paperTitle}}\par
      }
      \vskip 0.25cm
      \paperAuthors\par
      \vskip 0.35cm
      \paperAffiliations\par
      \vskip 0.08cm
      \paperNotes\par
    }
    \vskip 0.2cm
    {\color{textgray}%

Recent advances in large language models and programmatic CAD have significantly improved Text-to-CAD generation for individual parts. However, production-ready mechanical assembly generation remains largely unsolved. Unlike single-part modeling, assemblies require coordinated reasoning over multiple components, functional interfaces, assembly relations, engineering principles, and physical consistency. Consequently, directly generating executable CAD code is insufficient for constructing mechanically valid and reusable assemblies.

We present \textsc{AssemCAD}, an axiom-grounded framework for production-ready CAD assembly generation from natural language. Instead of representing an assembly as monolithic CAD code, \textsc{AssemCAD} first constructs an axiomatic Assembly Specification consisting of typed parts, geometry-backed ports, executable mates, and engineering axioms. Each assembly relation is explicitly grounded in one or more engineering principles, making the resulting specification interpretable, reusable, and verifiable. To realize this specification, \textsc{AssemCAD} introduces a port- and mate-based CAD assembly library that executes symbolic assembly relations through deterministic mate transformations and validates declared interfaces using concrete B-Rep geometric evidence.

Built on this representation and library, \textsc{AssemCAD} further supports on-demand synthesis of reusable parametric component factories for both standard and open-world geometries. A deterministic verification pipeline then checks interface validity, clash consistency, graph connectivity, degree-of-freedom constraints, and engineering-rule compliance, producing production-ready CAD assemblies together with interpretable verification reports. Experiments on \textsc{AssemBench} show that \textsc{AssemCAD} substantially improves assembly preservation and physical validity over code-centric CAD generation baselines, while generalizing across different foundation-model backbones. By combining axiom-grounded assembly reasoning with deterministic geometric execution, \textsc{AssemCAD} extends Text-to-CAD from isolated part generation toward production-ready mechanical assembly design.

\par}
    \vskip 0.4cm
    {
      \setlength{\parskip}{0cm}
      {\small {\sffamily\bfseries \raisebox{-0.2em}{\includegraphics[width=0.025\linewidth]{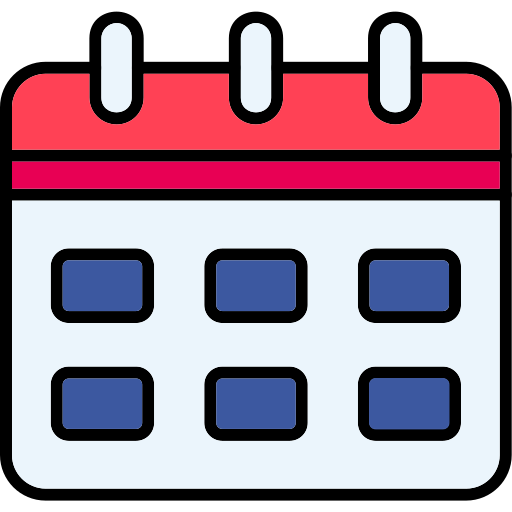}}~~Date:} \publishDate}\par%
    }
  \end{tcolorbox}
  \tcbset{reset}
}

\begin{document}

\newgeometry{top=1in, bottom=0.75in, textwidth=6.3in, textheight=9in}
\renderFrontBox

\begin{figure}[H]
  \centering
    \includegraphics[width=0.95\textwidth]{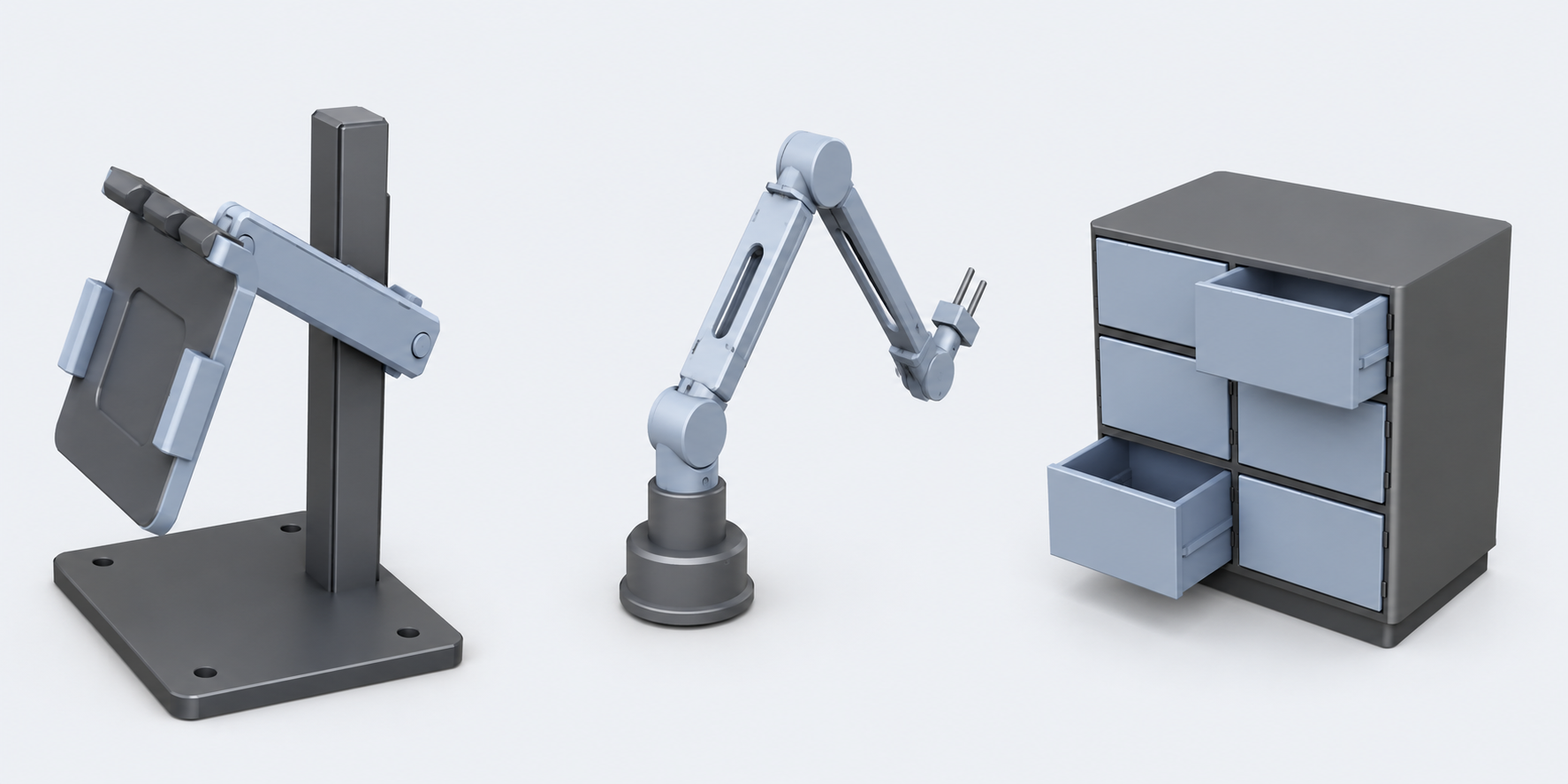}
    \caption{
    Representative production-ready CAD assemblies generated directly from natural language. \textbf{ASSEMCAD} generates mechanically executable assemblies with valid assembly structures and articulated joints, rather than merely static CAD geometry.
    }
    \label{fig:teaser}
\end{figure}



\section{Introduction}

Computer-aided design (CAD)~\cite{briere2012comparing} models are fundamental to modern engineering workflows, supporting downstream applications such as manufacturing, simulation, and assembly planning. While traditional CAD systems rely on interactive commercial software (e.g., SolidWorks)~\cite{lambourne2021brepnet,wu2021deepcad,verma2018autodesk}, where design intent is implicitly encoded in GUI operations and feature histories, programmatic CAD frameworks such as CadQuery and OpenSCAD~\cite{wright2024cadquery,machado2019parametric} provide explicit, parametric, and editable representations that naturally support procedural modeling and automated reasoning~\cite{xie2025text,li2024cad,barkley2026cadsmith,niu2026cme,hu2026itercad,ai2026comact}. Combined with recent advances in large language models (LLMs), substantial progress has been made in generating executable CAD models from natural language and visual inputs. Representative methods such as Text2CAD~\cite{khan2024text2cad}, PointCloud2CAD~\cite{liu2024point2cad}, and Cadrille~\cite{kolodiazhnyi2025cadrille} demonstrate strong capabilities for text-to-CAD, image-to-CAD, and multimodal CAD generation, while frontier code-centric foundation models (e.g., Claude and Codex~\cite{claudecode2025,openai2025codex,gemini35,glm2025glm45,qwen37}) further enable zero-shot executable CAD program synthesis. Together, these advances establish natural-language-driven CAD generation as a promising paradigm for automated engineering design.

Despite these advances, as illustrated in Figure~\ref{fig:intro}, production-ready\footnote{In this work, production-ready refers to assemblies that are executable, structurally connected, interface-valid, and free of unintended geometric interference, rather than fully optimized for manufacturing processes, materials, or tolerance analysis.} CAD assembly generation remains fundamentally more challenging than isolated part synthesis. Unlike standalone parts, mechanical assemblies are structured systems whose functionality depends not only on individual components but also on their interactions. This introduces three fundamental challenges. First, assemblies require explicit reasoning over functional relationships and geometric constraints across multiple components, rather than generating independent geometries. Second, the component space is inherently open-ended, requiring support for both standard parts and application-specific geometries beyond predefined templates. Third, correctness extends beyond executable CAD code or individually valid solids. Production-ready assemblies additionally require physically realizable interfaces, correct assembly relations, and interference-free configurations. These challenges fundamentally distinguish assembly generation from single-part CAD modeling and explain the gap between the left and right examples in Figure~\ref{fig:intro}.

These challenges suggest that production-ready CAD assembly generation should be viewed as an engineering pipeline rather than a code generation problem. As illustrated in Figure~\ref{fig:intro}, producing executable CAD code alone is insufficient, as failures may still arise from invalid components, incorrect assembly relations, or inconsistent assembly configurations. We therefore decompose the problem into three complementary stages, where foundation models perform semantic reasoning and structured planning, while programmatic CAD kernels provide deterministic geometry construction and physical verification. Specifically, the pipeline addresses three fundamental questions: what to build, how to build it, and how to verify the resulting assembly. This decomposition enables stage-wise validation and reliable production-ready assembly generation.

Guided by this insight, we present \textsc{AssemCAD}, a framework for production-ready CAD assembly generation from natural-language specifications. \textsc{AssemCAD} combines foundation models with programmatic CAD kernels to bridge semantic reasoning and deterministic geometric execution. Corresponding to the three stages of assembly generation, the framework comprises (1) a structured assembly specification for representing components, interfaces, and assembly semantics, (2) an on-demand component synthesis module for generating reusable parametric parts, and (3) a geometry-aware verification module for ensuring physical consistency and assembly correctness. Together, these components enable reliable production-ready CAD assembly generation through a transparent and verifiable engineering pipeline.

Our contributions are four-fold:

\begin{itemize}[leftmargin=*]

\item We present \textsc{AssemCAD}, a framework for production-ready CAD assembly generation from natural language. It formulates assembly generation as the construction of verifiable engineering specifications rather than direct CAD code generation.

\item We introduce an axiomatic assembly specification that explicitly represents typed parts, geometry-backed ports, executable mates, and engineering axioms. By grounding each assembly relation in explicit engineering principles, the specification provides an interpretable bridge between natural-language intent and geometric realization.

\item We develop a port- and mate-based CAD assembly library that operationalizes the specification through reusable assembly-ready components, deterministic mate transforms, and port-geometry consistency checks. This library supports both standard components and open-world synthesized parts while ensuring that declared interfaces are physically meaningful.

\item We establish \textsc{AssemBench} together with an assembly-centric evaluation protocol, and demonstrate through extensive experiments that \textsc{AssemCAD} substantially improves assembly preservation, physical validity, and generalization across different foundation models.

\end{itemize}


\begin{figure*}[t]
\centering
\includegraphics[width=\textwidth]{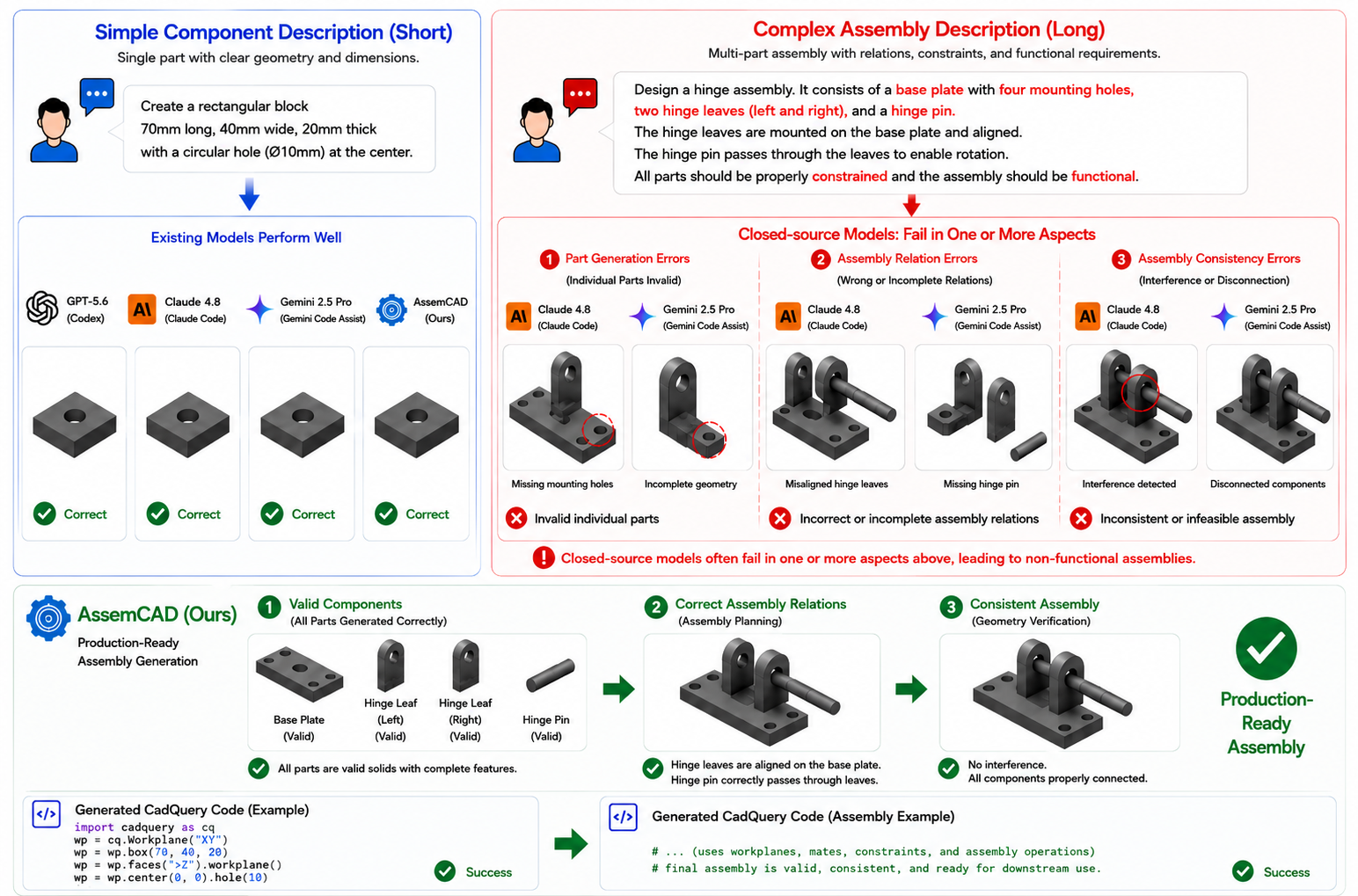}
\caption{
Motivation and overview of \textsc{AssemCAD}. Existing foundation models perform well on single-part CAD generation but struggle with production-ready assembly generation, where failures commonly arise from invalid components, incorrect assembly relations, and inconsistent assembly configurations. To address these challenges, \textsc{AssemCAD} decomposes assembly generation into three stages: structured assembly specification, on-demand component synthesis, and geometry-aware verification, enabling the generation of production-ready CAD assemblies from natural-language specifications.
}
\label{fig:intro}
\end{figure*}

\section{Related Work}

\subsection{Programmatic CAD Generation and B-Rep Learning}

Recent work on generative CAD has explored multiple representations, including construction sequences, executable CAD programs, and boundary representations. Sequence-based methods~\cite{wu2021deepcad, xu2022skexgen, xu2023hierarchical} model CAD construction histories autoregressively, while Fusion 360 Gallery~\cite{willis2021fusion360} provides human-authored sketch-and-extrude programs and an interactive reconstruction environment, and Free2CAD~\cite{li2022free2cad} parses freehand drawings into CAD command sequences. Text-conditioned and multimodal methods~\cite{wang2025cad, xu2024cad} further extend this paradigm to natural-language~\cite{xie2025text, khan2024text2cad, guan2026cad}, image~\cite{kolodiazhnyi2025cadrille}, and point-cloud inputs~\cite{liu2024point2cad, rukhovich2025cad}, producing editable CAD programs or command sequences. In parallel, B-Rep representation learning methods such as UV-Net~\cite{Jay2021uvnet} and BRepNet~\cite{lambourne2021brepnet} operate directly on parametric surfaces and CAD topology, while B-Rep reconstruction and generative models such as ComplexGen~\cite{guo2022complexgen}, SolidGen~\cite{jayaraman2023solidgenautoregressivemodeldirect}, and BrepGen~\cite{xu2024brepgenbrepgenerativediffusion} synthesize boundary representations without relying on construction-history supervision. These works establish strong foundations for part-level CAD generation and reconstruction. However, they primarily focus on individual solids or isolated construction histories, and do not explicitly model inter-part assembly relations, functional interfaces, or physical consistency required by production-ready assemblies.

\subsection{CAD Assembly Modeling and Relation Inference}

Mechanical assemblies are not merely collections of valid parts; their functionality depends on mates, contacts, kinematic constraints, and assembly topology~\cite{jones2021automatedatasetlearningapproach}. Classical constraint-based assembly modeling studies represent assemblies through geometric mating constraints and reason about under-, over-, and fully-constrained configurations, interference, and remaining degrees of freedom~\cite{ANANTHA1996707, zou2022review}. Recent learning-based methods revisit these problems using CAD data. JoinABLe~\cite{Willis2022Join} predicts parametric CAD joints from solid models and releases assembly data with joints, contact surfaces, holes, and assembly graphs, while Mates2Motion~\cite{noeckel2023mates2motionlearningmechanicalcad} learns degrees of freedom between mated CAD parts. Beyond joint prediction, previous work also investigates B-Rep-based automatic mate prediction~\cite{jones2021automatedatasetlearningapproach} and physically feasible assembly planning, often through assembly-by-disassembly~\cite{tian2022assemble, tian2023asap}. These methods typically assume that component geometries are already given and focus on recovering or predicting assembly relations. In contrast, AssemCAD generates both open-world components and their assembly specification from natural language, and verifies the resulting structure through deterministic geometric execution.

\subsection{Foundation Models and Agentic CAD Systems}

Recent foundation models~\cite{claudecode2025,openai2025codex,gemini35,glm2025glm45,qwen37} have demonstrated remarkable capabilities in code generation, spatial reasoning, and long-horizon planning, while agentic frameworks further improve reliability through iterative tool use, execution feedback, external memories, and reusable skills. These advances have naturally extended to CAD generation~\cite{xu2024cad}. Cad-assistant~\cite{mallis2025cad} formulates CAD interaction as a generic tool-augmented task-solving problem, where a vision-language model plans CAD-specific actions that are executed through FreeCAD and Python APIs. CADSmith~\cite{barkley2026cadsmith} improves text-to-CadQuery generation with a multi-agent pipeline and programmatic geometric validation, combining execution repair with CAD-kernel measurements such as bounding boxes, volumes, and solid validity. Solver-aided systems such as AIDL~\cite{jones2025solver} further offload low-level spatial reasoning to geometric constraint solvers through a hierarchical domain-specific language.

These systems demonstrate the promise of tool-augmented, execution-grounded, and solver-aided foundation models for CAD generation. However, their reliability is mainly achieved through local code repair, geometric validation, or solver-assisted construction. In contrast, AssemCAD improves practical reliability by coupling an assembly-oriented component library, a curated engineering axiom system, and a deterministic verification pipeline. This design grounds generation in reusable components and explicit port-mate semantics, while verifying clashes, connectivity, degrees of freedom, and assembly consistency, enabling more robust production-ready mechanical assembly generation.


Overall, prior work has advanced CAD generation, assembly modeling, and foundation-model-driven CAD systems from complementary perspectives. Programmatic and B-Rep methods improve part-level CAD synthesis, assembly modeling methods recover joints, mates, and motions from existing geometries, while agentic CAD systems enhance executability through tool use, repair, validation, or solver grounding. However, production-ready CAD design requires more than geometric plausibility. Prior work on knowledge-based engineering, ontology-based assembly design, and semantic mates shows that assembly relations carry engineering meaning beyond raw spatial transformations. Inspired by this line of work, AssemCAD associates each mate with engineering axioms covering foundations, constraints, kinematics, bearings, gears, fastening, hinges, and related categories. Unlike conventional knowledge-based systems that mainly support retrieval or post-hoc checking, AssemCAD integrates engineering semantics into the generative loop: axioms guide semantic decomposition, constrain component synthesis, and are propagated to the final verification report. Together with an assembly-oriented component library and deterministic verification pipeline, this enables AssemCAD to generate more reliable production-ready mechanical assemblies from natural language.

\section{Problem Formulation}
\label{sec:problem}

\paragraph{Traditional CAD Generation.}
Existing CAD generation methods typically formulate the task as directly translating textual or visual descriptions into geometric representations or executable CAD programs~\cite{khan2024text2cad,kolodiazhnyi2025cadrille}. Such formulations are well suited for individual part generation, where the output is a standalone geometry represented by a procedural CAD program or a sequence of construction operations. However, production-ready mechanical assemblies consist of multiple interacting components whose correctness depends not only on geometry, but also on interfaces, assembly relationships, and engineering semantics. In this work, we formulate production-ready CAD assembly generation as the generation of both a structured assembly specification and its corresponding geometric realization, enabling deterministic construction and engineering-aware verification.

\subsection{Assembly Formulation}

\begin{tcolorbox}[colback=gray!5,colframe=black!30,title={Core Formulation}]
Production-ready CAD assembly generation should not directly map natural-language descriptions to final geometry. Instead, it should first construct an intermediate representation that explicitly captures \textbf{components}, \textbf{interfaces}, \textbf{assembly relationships}, and \textbf{engineering semantics}. This intermediate representation enables \textbf{stage-wise construction} and \textbf{geometry-aware verification}, providing a principled bridge between semantic specifications and geometric realization.
\end{tcolorbox}

Following this formulation, we introduce an intermediate representation, termed \emph{Assembly Specification}, which serves as the bridge between semantic descriptions and geometric realization.

\begin{definition}[Assembly Specification]
\label{def:spec}
An assembly specification is defined as a triplet
\[
\mathcal{S}=(\mathcal{P},\mathcal{M},\mathcal{A}),
\]
where $\mathcal{P}$ denotes a set of typed parts, $\mathcal{M}$ denotes a set of typed mates describing assembly relationships, and $\mathcal{A}$ denotes the engineering semantics associated with these relationships.

\begin{itemize}[leftmargin=*,nosep]

\item
$\mathcal{P}=\{p_1,\ldots,p_n\}$ is a set of typed parts, where each part
\[
p_i=(\texttt{id}_i,f_i,\theta_i,\Pi_i)
\]
consists of a factory
$f_i\in\mathcal{F}$,
its parameters $\theta_i$,
and a set of ports $\Pi_i$.

\item
$\mathcal{M}=\{m_1,\ldots,m_k\}$ is a set of typed mates, where each mate
\[
m_j=(\tau_j,b_j,c_j,\omega_j)
\]
specifies a mate type
$\tau_j\in\mathcal{T}$,
two endpoints,
and optional mate parameters
$\omega_j$.

\item
$\mathcal{A}\subseteq\mathcal{A}_{\mathrm{axioms}}$
denotes the subset of engineering axioms that justify the assembly relationships.

\end{itemize}

Specifically, $\mathcal{A}$ captures the engineering semantics associated with assembly relationships. Its construction and role are described in the following subsection.

\end{definition}

\subsection{Engineering Semantics}

\begin{wraptable}{r}{0.50\linewidth}
\vspace{-1em}
\centering
\small
\caption{Engineering axiom categories ($|\mathcal A_{\mathrm{axioms}}|=62$, curated from 139 candidates).}
\label{tab:axioms}
\begin{tabular}{ccr|ccr}
\toprule
Code & Category & $n$ & Code & Category & $n$\\
\midrule
F & Foundations & 5 &
N & Structural & 7\\
C & Constraints & 5 &
P & Power/Shafts & 3\\
K & Kinematic & 5 &
S & Sequencing & 6\\
B & Bearings & 4 &
T & Fastening & 12\\
G & Gears & 9 &
H & Hinges & 6\\
\bottomrule
\end{tabular}
\vspace{-1em}
\end{wraptable}
Unlike conventional CAD systems that rely solely on geometric constraints, production-ready assemblies must also satisfy engineering semantics. For example, meshing gears require compatible transmission parameters, bearings require coaxial alignment, and fastening components must satisfy assembly constraints beyond pure geometry. To explicitly encode such knowledge, \textsc{AssemCAD} introduces an engineering knowledge base consisting of 62 engineering axioms organized into 10 mutually exclusive and collectively exhaustive (MECE) categories, summarized in Table~\ref{tab:axioms}. Of these, 41 are distilled from assembly textbooks through a three-phase LLM-assisted extraction pipeline, and 21 are synthesized to cover component families present in the CAD libraries but absent from the source corpus; the complete extraction methodology is detailed in Appendix~\ref{app:axioms}.

Rather than treating assembly relationships as purely geometric connections, each mate in the assembly specification is associated with one or more engineering axioms that justify its existence and configuration. These axioms provide a shared semantic layer throughout the entire pipeline: they guide semantic decomposition, support component synthesis, and are propagated to the final verification report. Consequently, assembly decisions are grounded not only in geometric validity, but also in interpretable engineering principles.

\subsection{Assembly Generation}

Based on the assembly specification and engineering semantics, the production-ready CAD assembly generation task is formally defined as follows.

\begin{definition}[Assembly Generation]
\label{def:assembly_gen}

Given a natural-language description
$d\in\Sigma^*$,
the goal is to produce both an assembly specification
$\mathcal{S}$
and its geometric realization
\[
G:\mathcal{S}\rightarrow\mathbb{R}^3,
\]
such that:

\begin{enumerate}[leftmargin=*,nosep]

\item
Each part
$p_i\in\mathcal{P}$
is instantiated as a valid B-Rep solid with positive volume.

\item
Every mate
$m_j\in\mathcal{M}$
is geometrically satisfiable under
$G$.

\item
The resulting assembly graph
$(\mathcal{P},\mathcal{M})$
is connected with a designated root component.

\item
No unintended interference exists, i.e., for any
$(p_i,p_j)\notin\mathcal{E}_{\mathrm{expected}}$,
the intersection volume between the corresponding solids is below a predefined threshold
$\tau$.

\end{enumerate}

\end{definition}
\section{AssemCAD}
\label{sec:method}

\subsection{System Overview}

\begin{figure*}[t]
\centering
\includegraphics[width=\textwidth]{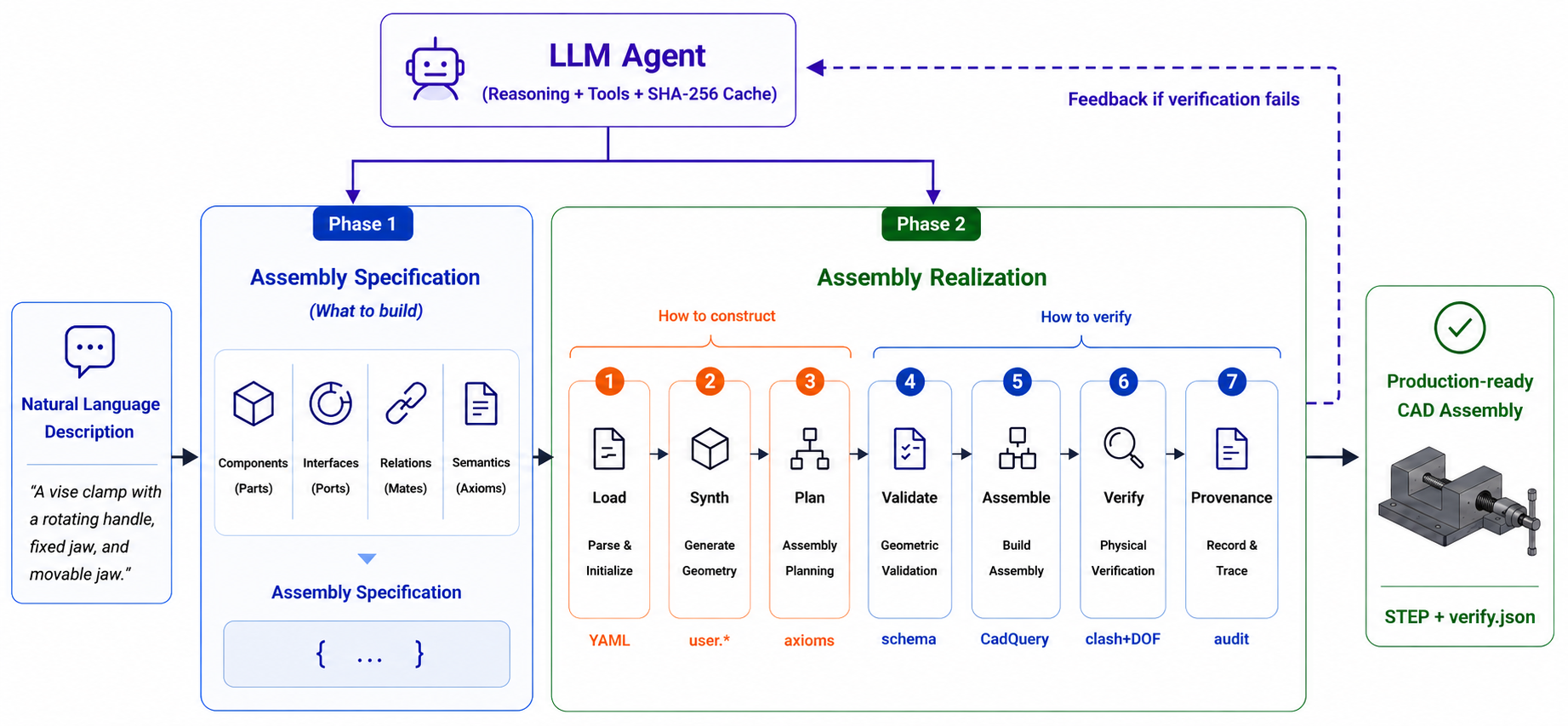}
\caption{
Overview of \textsc{AssemCAD}. The framework first transforms natural-language descriptions into a structured Assembly Specification, and then realizes the specification through an execution pipeline for geometry construction and engineering verification. Verification failures are fed back to the LLM agent for specification refinement, while successful executions produce production-ready CAD assemblies together with verification reports.
}
\label{fig:pipeline}
\end{figure*}

Building upon the problem formulation in Section~\ref{sec:problem}, \textsc{AssemCAD} decomposes production-ready CAD assembly generation into two complementary stages: \textbf{\emph{Assembly Specification}} and \textbf{\emph{Assembly Realization}}, as illustrated in Figure~\ref{fig:pipeline}. The framework first converts natural-language descriptions into a structured Assembly Specification using an LLM agent, and subsequently realizes the specification through a deterministic realization pipeline. Verification failures are fed back to the LLM agent for iterative refinement, while successful executions produce production-ready CAD assemblies together with verification reports. Unlike conventional end-to-end CAD generation systems, \textsc{AssemCAD} explicitly separates semantic reasoning from geometric realization, allowing each stage to be independently validated and refined through verification feedback.

\paragraph{\emph{Assembly Specification}.}
Given a natural-language description, the LLM agent generates a structured Assembly Specification that explicitly captures components, interfaces, assembly relationships, and engineering semantics, providing a semantic representation of the target assembly.

\paragraph{\emph{Assembly Realization}.}
Given the Assembly Specification, the realization pipeline progressively constructs, assembles, and verifies the corresponding CAD assembly through seven sequential stages, ensuring geometric validity, physical consistency, and traceable execution.

\begin{figure*}[t]
\centering
\begin{minipage}[t]{0.58\textwidth}
\vspace{0pt}
\begin{algorithm}[H]
\caption{Structured Decomposition with Bounded Repair}
\label{alg:decompose}
\small
\textbf{Require:} NL description $d$, LLM client $\mathcal{L}$, max rounds $K$\\
\textbf{Ensure:} Valid AssemblyDraft $\hat{S}$ or error
\begin{algorithmic}[1]
\State $\hat{S} \leftarrow \mathcal{L}(d)$
\For{$k=1,\dots,K$}
    \State $(valid,errors)\leftarrow \textsc{ValidateDraft}(\hat{S})$
    \If{$valid$}
        \State \Return $\hat{S}$
    \EndIf
    \State $\hat{S}\leftarrow \mathcal{L}(d,errors)$
\EndFor
\State \textbf{raise} DecompositionError
\end{algorithmic}
\end{algorithm}
\end{minipage}
\hfill
\begin{minipage}[t]{0.40\textwidth}
\vspace{10pt}
\centering
\small
\captionof{table}{Mate types with geometric semantics and configurable options.}
\label{tab:mate_types}
\vspace{-3pt}
\begin{tabular}{@{}lll@{}}
\toprule
\textbf{Mate Type $\tau$} &
\textbf{Constraint} &
\textbf{Options} \\
\midrule
\texttt{face\_to\_face}
& Coplanar normals
& $u,v,d_n$ \\
\texttt{coaxial}
& Shared axis
& $\alpha,d$ \\
\texttt{coaxial\_face}
& Axis + contact
& $d$ \\
\texttt{gear\_mesh}
& Pitch tangency
& $m,r$ \\
\texttt{press\_fit}
& Interference fit
& $d$ \\
\texttt{thread\_engage}
& Thread engagement
& $d_t$ \\
\texttt{snap\_to\_face}
& Surface snap
& $u,v,d_n$ \\
\bottomrule
\end{tabular}
\end{minipage}
\end{figure*}

\subsection{Phase I: What to Build}
The first phase determines \emph{what to build}. Given a natural-language description, \textsc{AssemCAD} performs semantic decomposition and constructs an \emph{Assembly Specification}, which serves as the interface between language understanding and the downstream deterministic pipeline. Rather than directly generating geometry, the system first identifies components, interfaces, assembly relationships, and engineering semantics, thereby separating semantic intent from geometric realization.

Given a natural-language description, the LLM decomposes the design intent into an \texttt{AssemblyDraft} containing component factories, typed ports, mate relationships, and engineering semantics. Since the generated specification may violate structural constraints, a bounded repair loop is employed to validate and iteratively refine the draft. Algorithm~\ref{alg:decompose} summarizes the decomposition procedure.

To explicitly model assembly interfaces, spatial relationships between parts are mediated through \emph{ports}—typed geometric attachment points declared by each factory. Unlike conventional CadQuery assemblies, where rigid transformations are directly specified between parts, \textsc{AssemCAD} represents connectivity through typed ports and mates. This interface-centric abstraction decouples assembly semantics from geometric realization and enables reusable attachment patterns across heterogeneous components. In our implementation, we define 12 port types and 7 mate types, summarized in Table~\ref{tab:mate_types}.

\begin{definition}[Port-Mate Compatibility]
\label{def:compat}
A mate of type $\tau\in\mathcal T$ is said to be compatible with a port pair $(\pi_b,\pi_c)$ if and only if $(t_b,t_c)\in C_\tau$, where
$C_\tau\subseteq\mathcal T_{\mathrm{port}}\times\mathcal T_{\mathrm{port}}$
denotes the compatibility relation associated with mate type $\tau$.
\end{definition}

Each mate specifies both geometric semantics and configurable parameters. For example, the \texttt{face\_to\_face} mate supports in-plane offsets $(u,v)$ and normal displacement $d_n$, allowing multiple components to reference the same base interface without geometric coincidence. This mechanism differs fundamentally from conventional rigid transformations and provides a structured foundation for subsequent construction and geometry-aware validation.

Beyond geometric semantics, each mate in the Assembly Specification is further grounded by referencing one or more engineering axioms from the axiom set $\mathcal{A}_{\mathrm{axioms}}$ introduced in Section~\ref{sec:problem}. During decomposition, the LLM is required to explicitly justify every mate by selecting appropriate axioms. For instance, a \texttt{coaxial} mate connecting a shaft to a bearing must reference axiom C-02 (clearance-fit constraint), while a \texttt{gear\_mesh} mate must reference axiom G-02 (module and pressure-angle compatibility). This axiom grounding serves three purposes: (1)~it constrains the LLM's assembly reasoning by requiring engineering justification rather than purely geometric plausibility; (2)~it makes the resulting specification interpretable, as each assembly relation carries an explicit engineering rationale traceable to the axiom set; and (3)~it enables downstream verification to check not only geometric consistency but also engineering-rule compliance. Once a valid and axiom-grounded Assembly Specification is obtained, the downstream pipeline is responsible for constructing and verifying the corresponding CAD assembly.

\subsection{Phase II: How to Construct}

The second phase determines \emph{how to construct individual components}. Given an Assembly Specification, \textsc{AssemCAD} transforms symbolic parts into executable geometry. This phase is built upon a port- and mate-based CAD assembly library that provides both a registry of reusable parametric component factories and an on-demand synthesis mechanism for open-world geometries.

\paragraph{Assembly Library and Factory Registry.}
The assembly library extends CadQuery with a typed port system and deterministic mate semantics (detailed in Appendix~\ref{app:cad_assembly_ext}). At its core, the library maintains a registry of 13 built-in parametric factories spanning common mechanical components including structural primitives (\texttt{rect\_plate}, \texttt{l\_bracket}, \texttt{u\_channel}, \texttt{housing\_block}), rotational components (\texttt{stepped\_shaft}, \texttt{coupling\_hub}, \texttt{boss}), and interface elements (\texttt{bolt\_circle\_flange}), etc. Additionally, the library provides adapter modules that wrap third-party parametric libraries---\texttt{cq\_warehouse} for standardized fasteners (bolts, nuts, washers) and bearings, and \texttt{cq\_gears} for involute spur gears, augmenting each component with typed port annotations. Every factory in the registry produces not only valid B-Rep geometry but also a set of typed ports that are compatible with the port-mate abstraction defined in Phase~I. This co-location of geometry construction and port declaration prevents the geometry--metadata drift that commonly arises when ports are annotated separately from their geometric features.

\paragraph{On-Demand Component Synthesis.}
For components beyond the built-in registry, \textsc{AssemCAD} supports open-world synthesis by generating previously unseen factories on demand via the LLM.
Each component is associated with a parameterized generator. During execution, \textsc{AssemCAD} first attempts to retrieve an existing implementation from the factory registry. If no matching implementation exists, the component is synthesized on demand by the LLM. The generated CadQuery module exposes both geometric parameters and the ports declared in Phase~I, thereby ensuring compatibility with the port-mate abstraction.

To avoid redundant synthesis, each request is identified by a content hash computed from the component description, expected ports, and parameter specification. Existing implementations are directly reused, whereas unseen requests trigger a synthesis procedure driven by the LLM.

Unlike template-based CAD generation methods, executable code alone is insufficient. Generated components must simultaneously satisfy geometric validity and interface correctness. Therefore, \textsc{AssemCAD} employs an iterative generation-repair loop equipped with a three-layer validation mechanism. Validation failures are converted into repair feedback and fed back to the LLM until a valid implementation is obtained or the retry budget is exhausted. Algorithm~\ref{alg:synthesis} summarizes the complete synthesis procedure.

\begin{algorithm}[t]
\caption{Factory Synthesis with Three-Layer Validation}
\label{alg:synthesis}
\small
\textbf{Require:} Description $\delta$, expected ports $\Pi_{\mathrm{exp}}$, parameters $\theta_0$, LLM client $\mathcal L$, max rounds $K$\\
\textbf{Ensure:} Validated factory module $\mathcal M$ or error
\begin{algorithmic}[1]
\State $k_{\mathrm{cache}}\gets \mathrm{SHA256}(\delta\Vert\Pi_{\mathrm{exp}}\Vert\theta_0)$
\If{$k_{\mathrm{cache}}\in\mathrm{Cache}$}
    \State \Return $\mathrm{Cache}[k_{\mathrm{cache}}]$
\EndIf
\State $code\gets\mathcal L.\mathrm{Chat}(\mathrm{SynthPrompt}(\delta,\Pi_{\mathrm{exp}},\theta_0))$
\For{$k=1,\dots,K$}
    \State $report\gets\mathrm{Sandbox}(code,\theta_0)$
    \If{$report.passed$}
        \State $\mathrm{Cache}[k_{\mathrm{cache}}]\gets code$
        \State \Return $code$
    \EndIf
    \State $code\gets\mathcal L.\mathrm{Chat}(\mathrm{RepairPrompt}(code,report))$
\EndFor
\State \textbf{raise} $\mathrm{SynthesisError}(report)$
\end{algorithmic}
\end{algorithm}

The first two validation layers ensure successful execution and geometric validity. The final layer verifies the consistency between declared ports and the generated geometry, ensuring that interfaces are supported by explicit geometric evidence rather than symbolic annotations alone.

\begin{definition}[Port-Geometry Consistency]
\label{def:consistency}
A port
\[
\pi=(name,t,\mathbf{o},\hat{\mathbf n},\phi)
\]
on solid $S$ satisfies \emph{port-geometry consistency} if and only if the type-specific geometric evidence holds:
\begin{equation}
\mathrm{Consistent}(\pi,S)\triangleq
\begin{cases}
\mathrm{CavityAt}(\mathbf{o},\hat{\mathbf n},\phi_d,S),
& t=\texttt{bore},\\
\mathrm{FaceAt}(\mathbf{o},\hat{\mathbf n},S),
& t=\texttt{flat\_face},\\
\mathrm{MaterialAlong}(\mathbf{o},\hat{\mathbf n},\phi_L,S),
& t=\texttt{shaft\_seat}.
\end{cases}
\label{eq:port_consistency}
\end{equation}
\end{definition}

where $\mathrm{CavityAt}$ requires sample points at radial offset $\phi_d/4$ to be classified as \texttt{OUT} by \texttt{BRepClass3d\_SolidClassifier}, $\mathrm{FaceAt}$ requires a coincident face with collinear normal (within $10^{-3}$ rad), and $\mathrm{MaterialAlong}$ requires \texttt{IN} classification at $\pm\phi_L/2$ along $\hat{\mathbf n}$.

This definition elevates ports from symbolic annotations to verifiable geometric interfaces. For example, a bore port is valid only when a cylindrical cavity exists at the declared location, while a shaft-seat port requires material support along the specified axis. Such consistency guarantees that assembly interfaces are grounded in geometry rather than merely declared in the specification.

Consequently, generator synthesis and interface validation are tightly coupled. Every admitted component is guaranteed not only to produce valid geometry, but also to expose physically meaningful interfaces that are compatible with downstream assembly operations.

\subsection{Phase III: How to Verify}

The third phase determines \emph{whether the generated assembly is production-ready}. Given validated component factories and an Assembly Specification, \textsc{AssemCAD} constructs the final assembly and subsequently performs geometry-aware verification to ensure physical consistency. This phase is realized through the assembly library's deterministic assembly engine, which operationalizes the port-mate abstraction through closed-form mate transforms and multi-layer verification.

\paragraph{Deterministic Mate Transform.}
Unlike conventional CAD assembly approaches that rely on iterative constraint solvers or manually specified rigid transformations, the assembly library computes each mate as a closed-form rigid-body transform. Given a base port $\pi_b$ (already placed in world coordinates) and an incoming port $\pi_c$ (in the incoming part's local frame), the transform that places the incoming part is:
\begin{equation}
\label{eq:mate_transform}
T = L_b \cdot R_{\mathrm{flip}} \cdot R_\alpha \cdot L_c^{-1},
\end{equation}
where $L_b, L_c \in SE(3)$ are the respective port frames, $R_{\mathrm{flip}}$ is a $180^\circ$ rotation about the local $x$-axis (making $z$-axes anti-parallel for face-to-face contact), and $R_\alpha$ is an optional user-specified rotation about the shared $z$-axis. Each of the seven mate types (Table~\ref{tab:mate_types}) extends this base transform with type-specific semantics: for example, \texttt{gear\_mesh} validates matching module and pressure angle before computing the center distance, while \texttt{press\_fit} verifies the interference condition before delegating to coaxial alignment. This single matrix multiplication produces byte-identical results across runs, satisfying Proposition~\ref{prop:determinism}, and eliminates the non-determinism inherent in iterative solvers.

\paragraph{Assembly Construction and Verification.}
Algorithm~\ref{alg:assembly} summarizes the complete assembly and verification process. Starting from instantiated parts, the system iteratively applies mate constraints through port alignment using Eq.~(\ref{eq:mate_transform}) and computes relative transformations between components. Before each mate is applied, the library checks port-type compatibility via the predefined compatibility matrix (Definition~\ref{def:compat}), catching category errors (e.g., mating a fastener port to a gear mesh port) before geometric computation. Once the assembly is constructed, pairwise interference analysis and connectivity verification are performed to generate the final report. The output consists of a STEP model together with a verification report.

\begin{algorithm}[t]
\caption{Assembly Construction and Verification}
\label{alg:assembly}
\small
\textbf{Require:} Validated tree $\mathcal T=(\mathcal P,\mathcal M)$, clash threshold $\tau$\\
\textbf{Ensure:} Assembly $\mathcal A$, verification report $\mathcal V$
\begin{algorithmic}[1]
\State $\mathcal A\gets \textsc{AssemblyExt}()$
\For{$p_i\in\mathcal P$}
    \State $w_i\gets f_i.\textsc{Build}(\theta_i)$
    \State $\mathcal A.\textsc{Add}(w_i,\texttt{id}_i)$
\EndFor
\For{$m_j=(\tau_j,b_j,c_j,\omega_j)\in\mathcal M$}
    \State $T_j\gets \textsc{ComputeTransform}(\tau_j,\mathcal A.\textsc{Port}(b_j),\mathcal A.\textsc{Port}(c_j),\omega_j)$
    \State $\mathcal A.\textsc{Locate}(c_j.\texttt{ref},T_j)$
\EndFor
\Statex \textit{// Verification}
\For{$(p_i,p_j)\in\binom{\mathcal P}{2}$}
    \State $v_{ij}\gets\mathrm{Vol}(\textsc{BRepCommon}(\mathcal A[p_i],\mathcal A[p_j]))$
    \If{$v_{ij}>\tau$}
        \State classify via Eq.~(\ref{eq:clash})
    \EndIf
\EndFor
\State $\mathcal V_{\mathrm{dof}}\gets\mathrm{BFS}(\mathcal A,\mathcal M,p_1)$
\State \Return $(\mathcal A,\mathcal V)$
\end{algorithmic}
\end{algorithm}

During assembly, boolean operations may silently produce disconnected compounds when solids only touch at boundaries. To prevent such failures, \textsc{AssemCAD} employs a robust union operator, \textsc{SafeUnion}, which explicitly preserves connectivity and raises errors whenever a valid fusion cannot be established.

\begin{proposition}[\textsc{SafeUnion} Correctness]
\label{prop:safeunion}
Let $B$ and $P$ be two CadQuery solids with
\[
|\mathrm{Solids}(B)|=|\mathrm{Solids}(P)|=1.
\]
Then $\textsc{SafeUnion}(B,P,\varepsilon)$ either returns a solid $R$ satisfying
\[
|\mathrm{Solids}(R)|=1,
\]
or raises an error. In particular, it never silently returns a disconnected compound.
\end{proposition}

\textit{Proof sketch.}
\textsc{SafeUnion} checks three cases sequentially. First, if
$\mathrm{Vol}(B\cap P)>0$,
standard union preserves connectivity. Second, if $B$ and $P$ share a common face, glue-based fusion is applied. Otherwise, one component is extended by a small tolerance $\varepsilon$ along the shortest-gap direction to create volumetric overlap before union. Finally, the post-condition
$|\mathrm{Solids}(R)|=1$
is explicitly asserted; any violation raises an error instead of propagating a disconnected compound.

Unlike conventional collision detection, interference analysis must be aware of assembly semantics. Certain mate types inherently require geometric overlap and therefore should not be classified as collisions. Specifically,
\[
\mathcal T_{\mathrm{contact}}
=
\{
\texttt{gear\_mesh},
\texttt{press\_fit},
\texttt{thread\_engage}
\}
\]
denotes mate types with intentional geometric interference. Consequently, clash classification is mate-type-aware:

\begin{equation}
\label{eq:clash}
\mathrm{Class}(p_i,p_j)=
\begin{cases}
\texttt{expected},
&
\exists m\in\mathcal M:
m\text{ connects }(p_i,p_j)
\land
\tau_m\in\mathcal T_{\mathrm{contact}},
\\[2pt]
\texttt{clash},
&
\text{otherwise}.
\end{cases}
\end{equation}

Therefore, geometric overlap does not necessarily indicate invalidity. Instead, interference is interpreted according to the semantics of the corresponding mate. This distinction enables \textsc{AssemCAD} to correctly handle assemblies involving gears, press fits, and threaded engagements, where contact and penetration are physically necessary.

Besides pairwise interference analysis, the verification stage also checks graph connectivity and degrees of freedom. Connectivity is verified by traversing the assembly graph, while remaining degrees of freedom are computed from the induced mate constraints. The resulting verification report records clashes, connectivity status, and mobility information, and is exported together with the final STEP model.

\begin{proposition}[Determinism]
\label{prop:determinism}
For any fixed input description $d$ and LLM response cache $\mathcal C$, the pipeline output $G(\mathcal S)$ is uniquely determined. Formally, if
\[
\mathcal C(q)=r
\]
for every query encountered during execution, then repeated invocations produce byte-identical STEP files and verification reports.
\end{proposition}

\textit{Proof sketch.}
The result follows directly from the content-addressed architecture of \textsc{AssemCAD}. LLM responses are cached by SHA256 hashes, generator synthesis is keyed by specification hashes, and all downstream computation is purely functional. Since assembly construction and verification contain no stochastic operations, identical inputs necessarily produce identical geometric realizations and identical reports.

\paragraph{Axiom Traceability.}
Beyond geometric verification, the axiom references attached to each mate during specification (Phase~I) are propagated through assembly construction to the final verification report. This creates an auditable chain from the output STEP file, through the assembly tree, back to the engineering axiom and ultimately to its textbook provenance. For example, if a clash is detected between two gears, the report not only identifies the interference volume but also references axiom G-02 (module compatibility), enabling engineers to trace whether the failure stems from a specification error or a geometric inconsistency. This axiom traceability distinguishes \textsc{AssemCAD} from conventional CAD pipelines, where verification results lack engineering semantics and are purely geometric in nature.

Consequently, \textsc{AssemCAD} transforms symbolic assembly specifications into reproducible CAD artifacts with explicit physical and engineering guarantees. The resulting STEP model and verification report together provide a production-ready representation suitable for downstream engineering and manufacturing workflows.
\section{Experiments}\label{sec:exper}

\subsection{Benchmark Construction}

Despite recent advances in natural-language CAD generation, there is currently no standardized evaluation protocol for production-ready CAD assembly generation. Existing CAD resources, including DeepCAD, Text2CAD, and Fusion360 Gallery~\cite{wu2021deepcad,khan2024text2cad,verma2018autodesk, willis2021fusion360}, are primarily designed for single-part modeling or data collection, and therefore do not evaluate assembly-level reasoning, inter-part relationships, or engineering correctness.
To address this gap, we introduce \textsc{AssemBench}, a benchmark built upon the Fusion360 Gallery assembly corpus. Instead of collecting a new dataset, \textsc{AssemBench} establishes a standardized evaluation protocol by curating mechanically meaningful assemblies and pairing them with unified natural-language assembly specifications generated through a semi-automatic annotation pipeline (Appendix~\ref{sec:benchmark_construction}).
The benchmark is constructed according to three principles: (1) \emph{engineering realism}, selecting functional mechanical assemblies with explicit engineering semantics; (2) \emph{structural complexity}, requiring multiple interacting components and non-trivial assembly relationships; and (3) \emph{parametric constructibility}, favoring assemblies that can be faithfully reconstructed using programmatic CAD operations.
As summarized in Table~\ref{tab:benchmark_stats}, \textsc{AssemBench} contains 120 assembly instances spanning diverse mechanical structures and varying levels of assembly complexity. It is specifically designed to evaluate the task formulation proposed in this work, requiring models to jointly synthesize components, recover assembly structures, and satisfy engineering constraints. Detailed benchmark construction, statistics, and representative examples are provided in Appendix~\ref{sec:benchmark_construction}.

\subsection{Evaluation Protocol}

We evaluate all methods under the two benchmark settings provided by \textsc{AssemBench}: \textbf{Text-to-CAD Assembly Generation} and \textbf{Image-Text-to-CAD Assembly Generation}. Since production-ready CAD assembly generation requires both executable geometry and correct assembly semantics, we evaluate each method using three complementary metrics: \textbf{Success Rate (SR)}, \textbf{Assembly Preservation Rate (APR)}, and \textbf{VLM-based Engineering Evaluation}.

\textbf{Chamfer Distance.} Although Chamfer Distance (CD) is widely adopted in prior 3D generation work, we deliberately exclude it from our primary evaluation. As we prove in Appendix~\ref{app:cd_unreliability}, CD is fundamentally unreliable for assembly evaluation: it lacks both \emph{assembly structure identifiability} (Proposition~\ref{prop:non_ident}) and \emph{assembly ranking consistency} (Theorem~\ref{thm:reversal}). In particular, a structurally correct assembly can receive a worse CD score than one missing functionally critical components (Eq.~\ref{eq:reversal}). We therefore adopt assembly-aware metrics that directly evaluate structural correctness and engineering quality.%

\textbf{Success Rate (SR).}
SR measures the percentage of benchmark instances that successfully generate executable CAD artifacts without runtime errors.
\textbf{Assembly Preservation Rate (APR).}
APR measures the percentage of executable CAD programs that preserve the intended assembly structure. Assembly preservation is evaluated using an LLM-based code judge, where predictions with scores of at least 4 (out of 5) are regarded as successful.
\textbf{VLM-based Engineering Evaluation.}
Beyond executability and assembly preservation, we evaluate engineering quality using a VLM-as-a-Judge. Given the input specification, the reference rendering, and the generated rendering, the VLM reports two scores: \textbf{Assembly}, measuring structural correctness, and \textbf{Overall}, measuring overall engineering quality.

\subsection{Main Results}
\begin{table*}[t]
\centering
\footnotesize
\setlength{\tabcolsep}{4.2pt}
\renewcommand{\arraystretch}{1.15}

\begin{tabular}{lcccc|lcccc}
\toprule

\multicolumn{5}{c|}{\textbf{Text-to-CAD Assembly Generation}}
&
\multicolumn{5}{c}{\textbf{Image-Text-to-CAD Assembly Generation}}
\\

\cmidrule(lr){1-5}
\cmidrule(lr){6-10}

\multirow{2}{*}{\textbf{Method}} &
\multirow{2}{*}{\textbf{SR}$\uparrow$} &
\multirow{2}{*}{\textbf{APR}$\uparrow$} &
\multicolumn{2}{c|}{\textbf{VLM Judge}} &
\multirow{2}{*}{\textbf{Method}} &
\multirow{2}{*}{\textbf{SR}$\uparrow$} &
\multirow{2}{*}{\textbf{APR}$\uparrow$} &
\multicolumn{2}{c}{\textbf{VLM Judge}}
\\
\cmidrule(lr){4-5}
\cmidrule(lr){9-10}  & & &
\textbf{Assembly}$\uparrow$ &
\textbf{Overall}$\uparrow$ & & & &
\textbf{Assembly}$\uparrow$ &
\textbf{Overall}$\uparrow$
\\

\midrule

Cadrille~\cite{kolodiazhnyi2025cadrille} & 65.00 & -- & -- & -- &
\modelicon{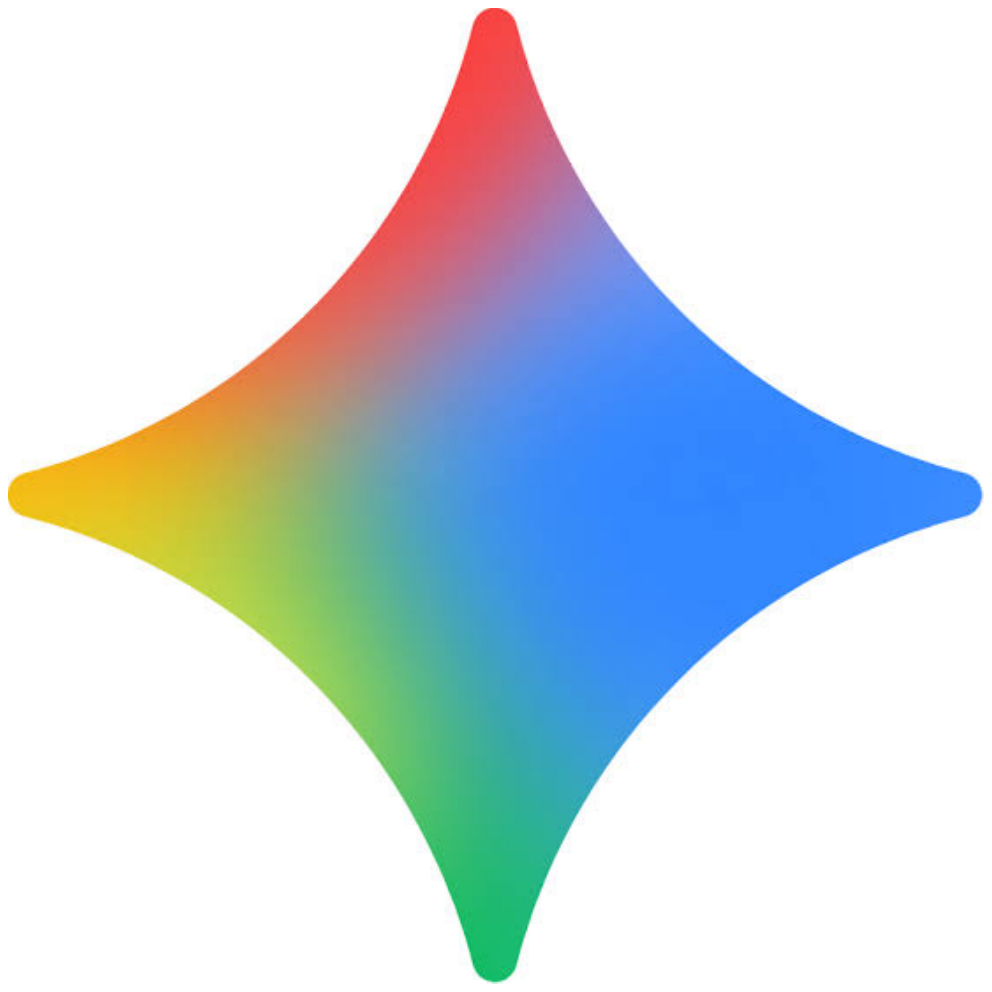}\hspace{2pt}Gemini 3.5 Flash & 30.83 & 0.00 & 3.21 & \textbf{5.41}
\\

\modelicon{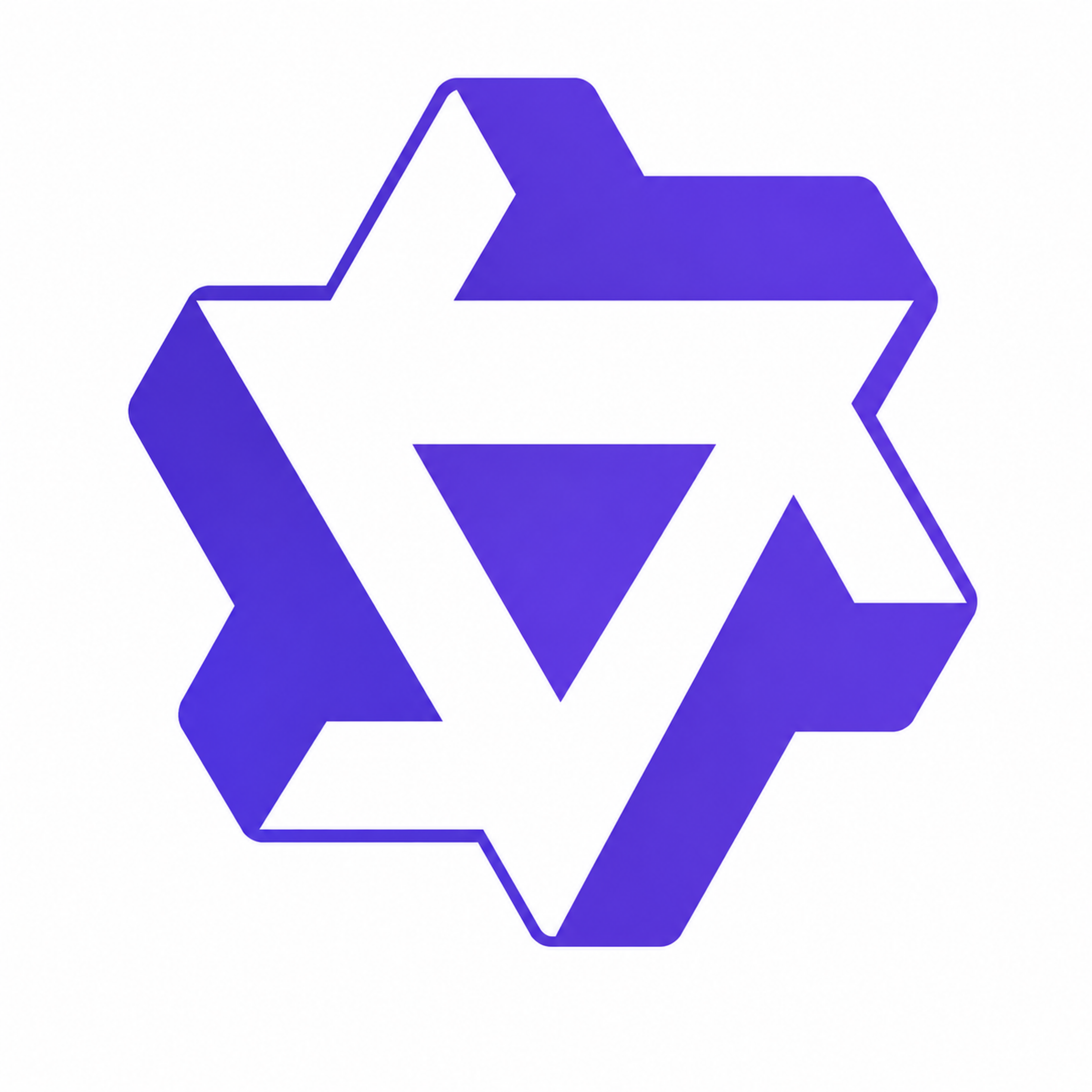}\hspace{2pt}Qwen3.7-Max & 78.33 & 39.17 & 2.88 & 4.53 &
\modelicon{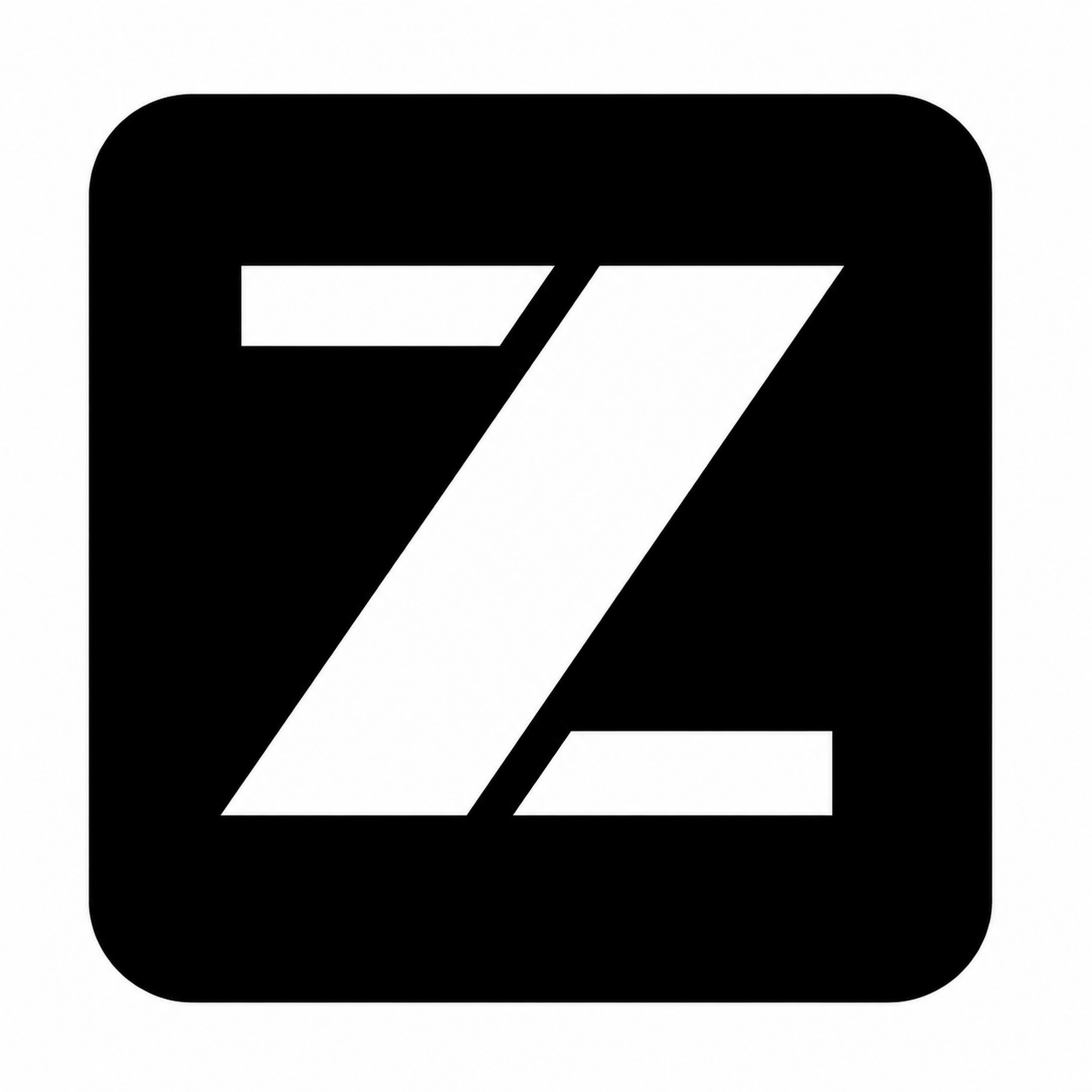}\hspace{2pt}GLM-5.2 & 76.67 & 40.83 & 2.65 & 4.11 
 \\

\modelicon{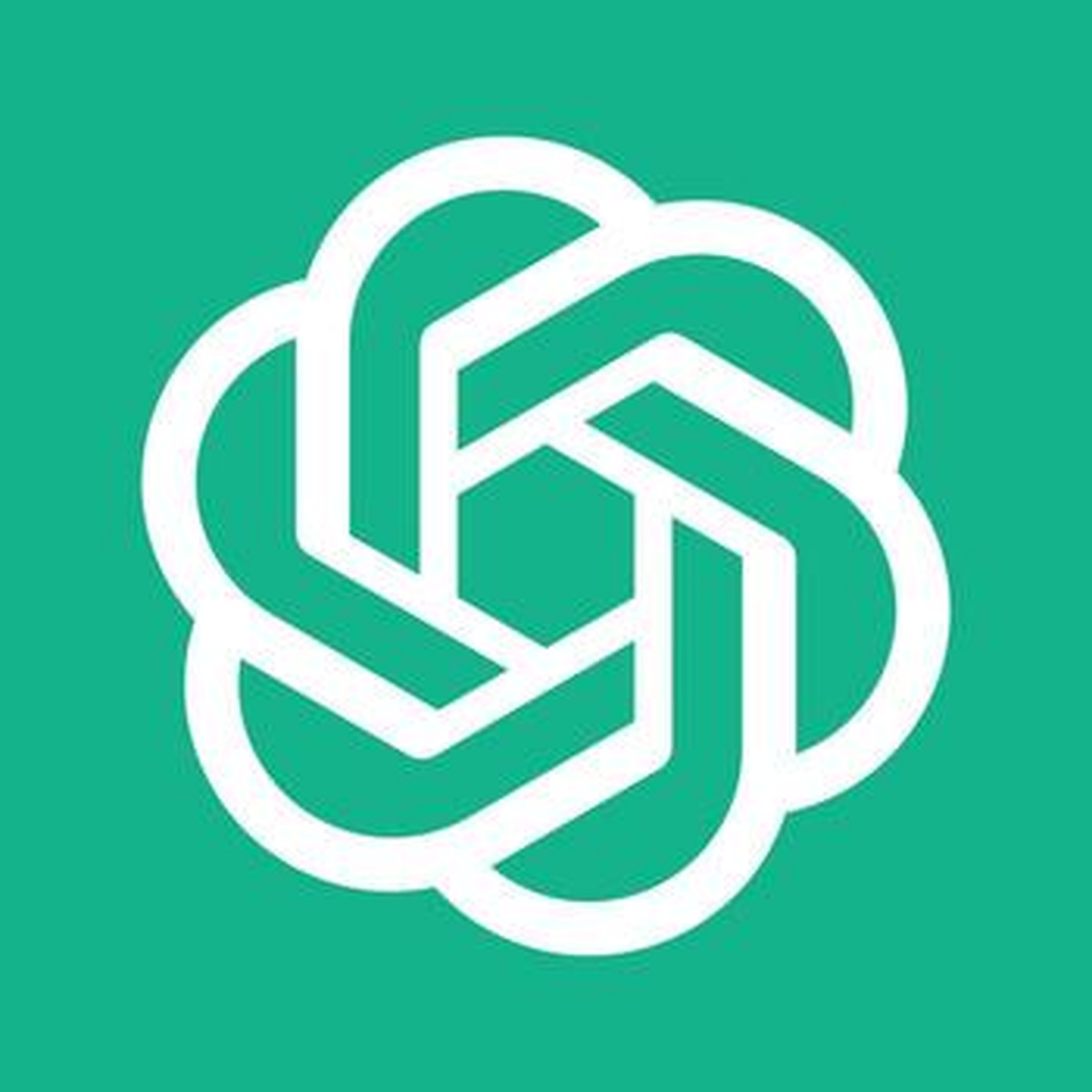}\hspace{2pt}GPT-5.5 & 79.17 & 66.67 & 2.96 & \textbf{4.73} &
\modelicon{images/logos/openai_logo.pdf}\hspace{2pt}GPT-5.5 & 79.17 & 64.17 & 3.17 & 5.20
\\

\modelicon{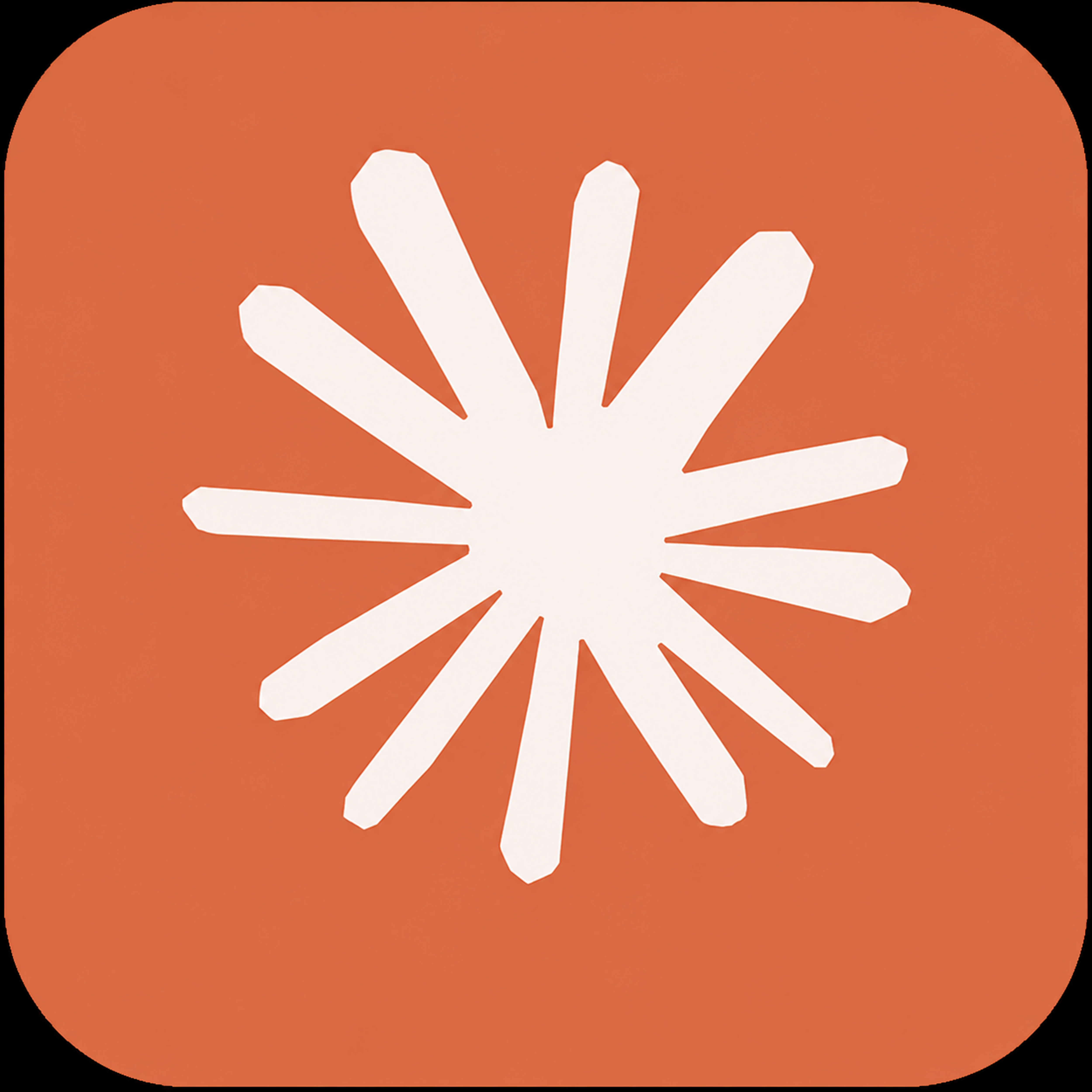}\hspace{2pt}Claude-4.8 & 85.00 & 0.83 & 2.70 & 4.37 &
\modelicon{images/logos/claude_logo.pdf}\hspace{2pt}Claude Opus 4.8 & 65.83 & 0.00 & 2.92 & 4.74  \\

\midrule

\rowcolor{gray!8}
\textbf{AssemCAD (Ours)} & \textbf{89.17} & \textbf{87.50} & \textbf{3.38} & 4.53 &
\textbf{AssemCAD (Ours)} & \textbf{84.17} & \textbf{82.50} & \textbf{3.48} & 4.77 \\

\bottomrule
\end{tabular}

\caption{
Comparison on \textsc{AssemBench} under two benchmark settings.
SR denotes \textbf{Success Rate}, APR denotes the proposed \textbf{Assembly Preservation Rate}, and the last two columns are evaluated using the proposed \textbf{VLM-as-a-Judge}. Higher is better for all metrics.
}
\label{tab:main_results}
\vspace{-3mm}
\end{table*}

Table~\ref{tab:main_results} summarizes the quantitative results on \textsc{AssemBench} under both text-only and image-only assembly generation settings. We compare \textsc{AssemCAD} with representative CAD generation methods and frontier foundation models. Overall, \textsc{AssemCAD} consistently achieves the highest Success Rate (SR) and Assembly Preservation Rate (APR), demonstrating its ability to reliably generate production-ready CAD assemblies while maintaining competitive engineering quality under the proposed VLM-as-a-Judge.

\paragraph{Text-to-CAD Assembly Generation.}
Under the text-only setting, frontier foundation models achieve relatively high execution success but exhibit a substantial drop in APR, indicating that many executable CAD programs fail to preserve the intended assembly structure. This SR--APR gap reveals a fundamental limitation of code-centric approaches: without explicit assembly semantics, executable programs may produce valid geometry while violating structural constraints such as port compatibility, mate consistency, or engineering axioms.

In contrast, \textsc{AssemCAD} achieves a SR of 89.17\% and an APR of 87.50\%, showing that nearly every successfully generated example also preserves valid assembly semantics. This near-perfect alignment between SR and APR is directly attributable to two core components. First, the port- and mate-based assembly library enforces structural correctness through typed port compatibility checks and deterministic mate transforms, ensuring that each assembly relation is geometrically realizable rather than merely syntactically valid. Second, the axiom-grounded specification requires every mate to reference explicit engineering principles, preventing semantically inconsistent assembly relations (e.g., mating a fastener to a gear mesh port) from entering the realization pipeline. Together, these mechanisms guarantee that successfully constructed assemblies are not only executable but also structurally and semantically faithful to the input specification. Furthermore, \textsc{AssemCAD} obtains the highest Assembly score, demonstrating superior structural correctness compared with both specialized CAD generation methods and general-purpose foundation models.

\paragraph{Image-Text-to-CAD Assembly Generation.}
Providing rendered images tends to improve visual and overall engineering judgments for several models. However, it does not consistently improve SR or APR, indicating that visual conditioning alone is insufficient for preserving assembly semantics. Nevertheless, preserving assembly semantics remains challenging, as reflected by the persistent gap between SR and APR. This gap suggests that visual perception can improve component-level geometry but provides limited guidance for assembly-level reasoning, where functional interfaces, mate compatibility, and engineering constraints must be explicitly modeled.

By explicitly modeling structured assembly specifications grounded in engineering axioms, \textsc{AssemCAD} again achieves the highest SR, APR, and Assembly score. The assembly library ensures that port-level interfaces are validated against concrete B-Rep geometry (Definition~\ref{def:consistency}), while axiom references attached to each mate provide engineering rationale that visual information alone cannot supply. This confirms that reliable assembly generation requires the combination of structured assembly representations and axiom-grounded reasoning, rather than relying solely on visual perception or code generation capabilities.

Although GPT-5.5 achieves the highest Overall score in the text-only setting and Gemini 3.5 Flash obtains the highest Overall score in the image-conditioned setting, these scores are computed only over successfully generated examples. Since \textsc{AssemCAD} successfully generates substantially more assemblies, it is evaluated on a considerably larger and more challenging set of instances, naturally making it more difficult to maintain the same average Overall score. Moreover, the assembly library's multi-layer verification pipeline (port-geometry consistency, clash detection, and connectivity analysis) admits assemblies that other methods would fail to produce, including structurally complex instances with intricate mate configurations. Nevertheless, \textsc{AssemCAD} remains highly competitive while consistently producing production-ready assemblies across the full benchmark.

\subsection{Framework Generalization}

\begin{wraptable}{r}{0.43\linewidth}
\vspace{-3mm}
\centering
\small
\setlength{\tabcolsep}{5pt}
\renewcommand{\arraystretch}{1.12}

\caption{Generalization across different foundation models.}
\label{tab:backbone}

\begin{tabular}{lccc}
\toprule
\textbf{Method}
&
\textbf{APR}$\uparrow$
&
\textbf{Assembly}$\uparrow$
&
\textbf{Overall}$\uparrow$
\\

\midrule

\modelicon{images/logos/claude_logo.pdf}\hspace{1pt}Claude-4.8
&
0.8
&
2.70
&
4.37
\\

\rowcolor{gray!10}
+ AssemCAD
&
\textbf{82.50}
&
\textbf{3.38}
&
\textbf{4.53}
\\

\midrule

\modelicon{images/logos/openai_logo.pdf}\hspace{1pt}GPT-5.5
&
64.2
&
3.17
&
\textbf{5.20}
\\

\rowcolor{gray!10}
+ AssemCAD
&
\textbf{71.67}
&
\textbf{3.58}
&
4.92
\\

\bottomrule
\end{tabular}

\vspace{-4mm}
\end{wraptable}

For Claude-4.8, AssemCAD dramatically improves both the APR and the Assembly Structure score under the text-only setting. Similarly, when combined with GPT-5.5, AssemCAD consistently improves assembly validity and structural quality under the image-conditioned setting.

This backbone-agnostic improvement can be attributed to the design of the assembly library and axiom system. The port- and mate-based library provides a deterministic realization layer that is independent of which foundation model generates the specification: typed port compatibility, deterministic mate transforms (Eq.~\ref{eq:mate_transform}), and port-geometry consistency checks (Definition~\ref{def:consistency}) operate identically regardless of the upstream reasoning model. Similarly, the 62 engineering axioms serve as a shared semantic contract that constrains the specification space, reducing the burden on the foundation model from generating correct assembly code to selecting appropriate components, ports, and axiom-grounded mates. Consequently, the framework is largely orthogonal to the underlying foundation model and can be readily integrated with different reasoning backbones.

\clearpage
{
\bibliographystyle{unsrt}  
\bibliography{preprint}

@article{briere2012comparing,
  title={Comparing 3D CAD models: uses, methods, tools and perspectives},
  author={Bri{\`e}re-C{\^o}t{\'e}, Antoine and Rivest, Louis and Maranzana, Roland},
  journal={Computer-Aided Design and Applications},
  volume={9},
  number={6},
  pages={771--794},
  year={2012},
  publisher={Taylor \& Francis}
}

@inproceedings{lambourne2021brepnet,
  title={Brepnet: A topological message passing system for solid models},
  author={Lambourne, Joseph G and Willis, Karl DD and Jayaraman, Pradeep Kumar and Sanghi, Aditya and Meltzer, Peter and Shayani, Hooman},
  booktitle={Proceedings of the IEEE/CVF conference on computer vision and pattern recognition},
  pages={12773--12782},
  year={2021}
}

@inproceedings{wu2021deepcad,
  title={Deepcad: A deep generative network for computer-aided design models},
  author={Wu, Rundi and Xiao, Chang and Zheng, Changxi},
  booktitle={Proceedings of the IEEE/CVF international conference on computer vision},
  pages={6772--6782},
  year={2021}
}

@book{verma2018autodesk,
  title={Autodesk fusion 360 black book},
  author={Verma, Gaurav},
  year={2018},
  publisher={BPB Publications}
}

@article{wright2024cadquery,
  title={CadQuery/cadquery: CadQuery 2.4. 0},
  author={Wright, Jeremy and Boyd, Marcus and {\"O}ZDERYA, Hasan Yavuz and Agostini, Bruno and Greminger, Michael and Fischer, Seth and Buchanan, Justin and S{\'a}nchez de Le{\'o}n Peque, Miguel and Budden, Martin and Boin, Peter and others},
  journal={Zenodo},
  year={2024}
}

@article{machado2019parametric,
  title={Parametric CAD modeling for open source scientific hardware: Comparing OpenSCAD and FreeCAD Python scripts},
  author={Machado, Felipe and Malpica, Norberto and Borromeo, Susana},
  journal={Plos one},
  volume={14},
  number={12},
  pages={e0225795},
  year={2019},
  publisher={Public Library of Science San Francisco, CA USA}
}

@article{xie2025text,
  title={Text-to-cadquery: A new paradigm for cad generation with scalable large model capabilities},
  author={Xie, Haoyang and Ju, Feng},
  journal={arXiv preprint arXiv:2505.06507},
  year={2025}
}

@inproceedings{li2024cad,
  title={Cad translator: An effective drive for text to 3d parametric computer-aided design generative modeling},
  author={Li, Xueyang and Song, Yu and Lou, Yunzhong and Zhou, Xiangdong},
  booktitle={Proceedings of the 32nd ACM International Conference on Multimedia},
  pages={8461--8470},
  year={2024}
}

@inproceedings{niu2026cme,
  title={Cme-cad: Heterogeneous collaborative multi-expert reinforcement learning for cad code generation},
  author={Niu, Ke and Yu, Haiyang and Chen, Zhuofan and Yao, Zhengtao and Jia, Weitao and Ge, Xiaodong and Tang, Jingqun and Cui, Benlei and Li, Bin and Xue, Xiangyang},
  booktitle={Proceedings of the IEEE/CVF Conference on Computer Vision and Pattern Recognition},
  pages={39272--39281},
  year={2026}
}

@article{guan2026cad,
  title={Cad-coder: Text-to-cad generation with chain-of-thought and geometric reward},
  author={Guan, Yandong and Wang, Xilin and Xing, Ximing and Zhang, Jing and Xu, Dong and Yu, Qian},
  journal={Advances in Neural Information Processing Systems},
  volume={38},
  pages={59765--59789},
  year={2026}
}

@article{kolodiazhnyi2025cadrille,
  title={cadrille: Multi-modal cad reconstruction with online reinforcement learning},
  author={Kolodiazhnyi, Maksim and Tarasov, Denis and Zhemchuzhnikov, Dmitrii and Nikulin, Alexander and Zisman, Ilya and Vorontsova, Anna and Konushin, Anton and Kurenkov, Vladislav and Rukhovich, Danila},
  journal={arXiv preprint arXiv:2505.22914},
  year={2025}
}

@inproceedings{liu2024point2cad,
  title={Point2cad: Reverse engineering cad models from 3d point clouds},
  author={Liu, Yujia and Obukhov, Anton and Wegner, Jan Dirk and Schindler, Konrad},
  booktitle={Proceedings of the IEEE/CVF conference on computer vision and pattern recognition},
  pages={3763--3772},
  year={2024}
}

@article{hu2026itercad,
  title={IterCAD: An Iterative Multimodal Agent for Visually-Grounded CAD Generation and Editing},
  author={Hu, Tao and Ai, Jiaxin and Wen, Licheng and Li, Xueheng and Zou, Shu and Li, Siqi and Deng, Nianchen and Cai, Xinyu and Zhou, Hongbin and Cai, Pinlong and others},
  journal={arXiv preprint arXiv:2606.13368},
  year={2026}
}

@article{ai2026comact,
  title={ComAct: Reframing Professional Software Manipulation via COM-as-Action Paradigm},
  author={Ai, Jiaxin and Hu, Tao and Yang, Xuemeng and Zou, Shu and Zhang, Hairong and Fu, Daocheng and Yang, Yu and Zhou, Hongbin and Deng, Nianchen and Cai, Pinlong and others},
  journal={arXiv preprint arXiv:2606.13239},
  year={2026}
}

@article{khan2024text2cad,
  title={Text2cad: Generating sequential cad designs from beginner-to-expert level text prompts},
  author={Khan, Mohammad S and Sinha, Sankalp and Sheikh, Talha U and Stricker, Didier and Ali, Sk A and Afzal, Muhammad Z},
  journal={Advances in Neural Information Processing Systems},
  volume={37},
  pages={7552--7579},
  year={2024}
}

@misc{claudecode2025,
  author       = {{Anthropic}},
  title        = {Claude Code},
  year         = {2025},
  howpublished = {\url{https://docs.anthropic.com/en/docs/claude-code}},
  note         = {Accessed: 2026-06-30}
}

@misc{openai2025codex,
  author = {{OpenAI}},
  title = {Introducing Codex},
  year = {2025},
  howpublished = {\url{https://openai.com/index/introducing-codex/}},
  note = {Accessed: 2026-06-30}
}

@article{gemini35,
  title={Gemini 2.5: Pushing the Frontier with Advanced Reasoning and Multimodality},
  author={Comanici, George and others},
  journal={arXiv preprint arXiv:2507.06261},
  year={2025}
}

@article{glm2025glm45,
  title={GLM-4.5: Agentic, Reasoning, and Coding (ARC) Foundation Models},
  author={{GLM-4.5 Team}},
  journal={arXiv preprint arXiv:2508.06471},
  year={2025}
}

@misc{qwen37,
  author = {{Alibaba Qwen Team}},
  title = {Qwen3.7},
  year = {2026},
  howpublished = {\url{https://chat.qwen.ai/}}
}

@inproceedings{wang2025cad,
  title={CAD-GPT: Synthesising CAD construction sequence with spatial reasoning-enhanced multimodal LLMs},
  author={Wang, Siyu and Chen, Cailian and Le, Xinyi and Xu, Qimin and Xu, Lei and Zhang, Yanzhou and Yang, Jie},
  booktitle={Proceedings of the AAAI Conference on Artificial Intelligence},
  volume={39},
  number={8},
  pages={7880--7888},
  year={2025}
}

@article{xu2024cad,
  title={Cad-mllm: Unifying multimodality-conditioned cad generation with mllm},
  author={Xu, Jingwei and Wang, Chenyu and Zhao, Zibo and Liu, Wen and Ma, Yi and Gao, Shenghua},
  journal={arXiv preprint arXiv:2411.04954},
  year={2024}
}

@inproceedings{rukhovich2025cad,
  title={Cad-recode: Reverse engineering cad code from point clouds},
  author={Rukhovich, Danila and Dupont, Elona and Mallis, Dimitrios and Cherenkova, Kseniya and Kacem, Anis and Aouada, Djamila},
  booktitle={Proceedings of the IEEE/CVF International Conference on Computer Vision},
  pages={9801--9811},
  year={2025}
}

@INPROCEEDINGS{Jay2021uvnet,
  author={Jayaraman, Pradeep Kumar and Sanghi, Aditya and Lambourne, Joseph G. and Willis, Karl D. D. and Davies, Thomas and Shayani, Hooman and Morris, Nigel},
  booktitle={2021 IEEE/CVF Conference on Computer Vision and Pattern Recognition (CVPR)}, 
  title={UV-Net: Learning from Boundary Representations}, 
  year={2021},
  volume={},
  number={},
  pages={11698-11707},
  doi={10.1109/CVPR46437.2021.01153}
}

@misc{jayaraman2023solidgenautoregressivemodeldirect,
      title={SolidGen: An Autoregressive Model for Direct B-rep Synthesis}, 
      author={Pradeep Kumar Jayaraman and Joseph G. Lambourne and Nishkrit Desai and Karl D. D. Willis and Aditya Sanghi and Nigel J. W. Morris},
      year={2023},
      eprint={2203.13944},
      archivePrefix={arXiv},
      primaryClass={cs.LG},
      url={https://arxiv.org/abs/2203.13944}, 
}

@article{xu2024brepgenbrepgenerativediffusion,
author = {Xu, Xiang and Lambourne, Joseph and Jayaraman, Pradeep and Wang, Zhengqing and Willis, Karl and Furukawa, Yasutaka},
title = {BrepGen: A B-rep Generative Diffusion Model with Structured Latent Geometry},
year = {2024},
issue_date = {July 2024},
publisher = {Association for Computing Machinery},
address = {New York, NY, USA},
volume = {43},
number = {4},
issn = {0730-0301},
url = {https://doi.org/10.1145/3658129},
doi = {10.1145/3658129},
journal = {ACM Trans. Graph.},
month = jul,
articleno = {119},
numpages = {14}
}

@INPROCEEDINGS{Willis2022Join,
  author={Willis, Karl D.D. and Jayaraman, Pradeep Kumar and Chu, Hang and Tian, Yunsheng and Li, Yifei and Grandi, Daniele and Sanghi, Aditya and Tran, Linh and Lambourne, Joseph G. and Solar-Lezama, Armando and Matusik, Wojciech},
  booktitle={2022 IEEE/CVF Conference on Computer Vision and Pattern Recognition (CVPR)}, 
  title={JoinABLe: Learning Bottom-up Assembly of Parametric CAD Joints}, 
  year={2022},
  volume={},
  number={},
  pages={15828-15839},
  doi={10.1109/CVPR52688.2022.01539}
}

@misc{noeckel2023mates2motionlearningmechanicalcad,
      title={Mates2Motion: Learning How Mechanical CAD Assemblies Work}, 
      author={James Noeckel and Benjamin T. Jones and Karl Willis and Brian Curless and Adriana Schulz},
      year={2023},
      eprint={2208.01779},
      archivePrefix={arXiv},
      primaryClass={cs.CV},
      url={https://arxiv.org/abs/2208.01779}, 
}

@article{jones2021automatedatasetlearningapproach,
author = {Jones, Benjamin and Hildreth, Dalton and Chen, Duowen and Baran, Ilya and Kim, Vladimir G. and Schulz, Adriana},
title = {AutoMate: a dataset and learning approach for automatic mating of CAD assemblies},
year = {2021},
issue_date = {December 2021},
publisher = {Association for Computing Machinery},
address = {New York, NY, USA},
volume = {40},
number = {6},
issn = {0730-0301},
url = {https://doi.org/10.1145/3478513.3480562},
doi = {10.1145/3478513.3480562},
journal = {ACM Trans. Graph.},
month = dec,
articleno = {227},
numpages = {18}
}

@article{ANANTHA1996707,
title = {Assembly modelling by geometric constraint satisfaction},
journal = {Computer-Aided Design},
volume = {28},
number = {9},
pages = {707-722},
year = {1996},
issn = {0010-4485},
doi = {https://doi.org/10.1016/0010-4485(96)00001-2},
url = {https://www.sciencedirect.com/science/article/pii/0010448596000012},
author = {Ram Anantha and Glenn A Kramer and Richard H Crawford}
}

@article{tian2022assemble,
place = {Country unknown/Code not available}, title = {Assemble Them All: Physics-Based Planning for Generalizable Assembly by Disassembly}, url = {https://par.nsf.gov/biblio/10419453}, DOI = {10.1145/3550454.3555525}, abstractNote = {Assembly planning is the core of automating product assembly, maintenance, and recycling for modern industrial manufacturing. Despite its importance and long history of research, planning for mechanical assemblies when given the final assembled state remains a challenging problem. This is due to the complexity of dealing with arbitrary 3D shapes and the highly constrained motion required for real-world assemblies. In this work, we propose a novel method to efficiently plan physically plausible assembly motion and sequences for real-world assemblies. Our method leverages the assembly-by-disassembly principle and physics-based simulation to efficiently explore a reduced search space. To evaluate the generality of our method, we define a large-scale dataset consisting of thousands of physically valid industrial assemblies with a variety of assembly motions required. Our experiments on this new benchmark demonstrate we achieve a state-of-the-art success rate and the highest computational efficiency compared to other baseline algorithms. Our method also generalizes to rotational assemblies (e.g., screws and puzzles) and solves 80-part assemblies within several minutes.}, journal = {ACM Transactions on Graphics}, volume = {41}, number = {6}, author = {Tian, Yunsheng and Xu, Jie and Li, Yichen and Luo, Jieliang and Sueda, Shinjiro and Li, Hui and Willis, Karl D. and Matusik, Wojciech}, }

@INPROCEEDINGS{tian2023asap,
  author={Tian, Yunsheng and Willis, Karl D. D. and Al Omari, Bassel and Luo, Jieliang and Ma, Pingchuan and Li, Yichen and Javid, Farhad and Gu, Edward and Jacob, Joshua and Sueda, Shinjiro and Li, Hui and Chitta, Sachin and Matusik, Wojciech},
  booktitle={2024 IEEE International Conference on Robotics and Automation (ICRA)}, 
  title={ASAP: Automated Sequence Planning for Complex Robotic Assembly with Physical Feasibility}, 
  year={2024},
  volume={},
  number={},
  pages={4380-4386},
  doi={10.1109/ICRA57147.2024.10611595}}

@article{willis2021fusion360,
author = {Willis, Karl D. D. and Pu, Yewen and Luo, Jieliang and Chu, Hang and Du, Tao and Lambourne, Joseph G. and Solar-Lezama, Armando and Matusik, Wojciech},
title = {Fusion 360 gallery: a dataset and environment for programmatic CAD construction from human design sequences},
year = {2021},
issue_date = {August 2021},
publisher = {Association for Computing Machinery},
address = {New York, NY, USA},
volume = {40},
number = {4},
issn = {0730-0301},
url = {https://doi.org/10.1145/3450626.3459818},
doi = {10.1145/3450626.3459818},
journal = {ACM Trans. Graph.},
month = jul,
articleno = {54},
numpages = {24}
}

@inproceedings{mallis2025cad,
  title={CAD-assistant: tool-augmented vllms as generic cad task solvers},
  author={Mallis, Dimitrios and Karadeniz, Ahmet Serda and Cavada, Sebastian and Rukhovich, Danila and Foteinopoulou, Niki and Cherenkova, Kseniya and Kacem, Anis and Aouada, Djamila},
  booktitle={Proceedings of the IEEE/CVF International Conference on Computer Vision},
  pages={7284--7294},
  year={2025}
}

@article{barkley2026cadsmith,
  title={Cadsmith: Multi-agent cad generation with programmatic geometric validation},
  author={Barkley, Jesse and Loghmani, Rumi and Farimani, Amir Barati},
  journal={arXiv preprint arXiv:2603.26512},
  year={2026}
}

@inproceedings{jones2025solver,
  title={A Solver-Aided Hierarchical Language for LLM-Driven CAD Design},
  author={Jones, Benjamin T and Zhang, Zihan and H{\"a}hnlein, Felix and Matusik, Wojciech and Ahmad, Maaz and Kim, Vladimir and Schulz, Adriana},
  booktitle={Computer Graphics Forum},
  volume={44},
  number={7},
  pages={e70250},
  year={2025},
  organization={Wiley Online Library}
}

@inproceedings{xu2022skexgen,
  title={SkexGen: Autoregressive Generation of CAD Construction Sequences with Disentangled Codebooks},
  author={Xu, Xiang and Willis, Karl DD and Lambourne, Joseph G and Cheng, Chin-Yi and Jayaraman, Pradeep Kumar and Furukawa, Yasutaka},
  booktitle={International Conference on Machine Learning},
  pages={24698--24724},
  year={2022},
  organization={PMLR}
}

@inproceedings{xu2023hierarchical,
  title={Hierarchical Neural Coding for Controllable CAD Model Generation},
  author={Xu, Xiang and Jayaraman, Pradeep Kumar and Lambourne, Joseph George and Willis, Karl DD and Furukawa, Yasutaka},
  booktitle={International Conference on Machine Learning},
  pages={38443--38461},
  year={2023},
  organization={PMLR}
}

@article{li2022free2cad,
author = {Li, Changjian and Pan, Hao and Bousseau, Adrien and Mitra, Niloy J.},
title = {Free2CAD: parsing freehand drawings into CAD commands},
year = {2022},
issue_date = {July 2022},
publisher = {Association for Computing Machinery},
address = {New York, NY, USA},
volume = {41},
number = {4},
issn = {0730-0301},
url = {https://doi.org/10.1145/3528223.3530133},
doi = {10.1145/3528223.3530133},
journal = {ACM Trans. Graph.},
month = jul,
articleno = {93},
numpages = {16}
}

@article{guo2022complexgen,
  title={Complexgen: Cad reconstruction by b-rep chain complex generation},
  author={Guo, Haoxiang and Liu, Shilin and Pan, Hao and Liu, Yang and Tong, Xin and Guo, Baining},
  journal={ACM Transactions on Graphics (TOG)},
  volume={41},
  number={4},
  pages={1--18},
  year={2022},
  publisher={ACM New York, NY, USA}
}

@article{zou2022review,
  title={A review on geometric constraint solving},
  author={Zou, Qiang and Tang, Zhihong and Feng, Hsi-Yung and Gao, Shuming and Zhou, Chenchu and Liu, Yusheng},
  journal={arXiv preprint arXiv:2202.13795},
  year={2022}
}
}

\clearpage
\newgeometry{
  textheight=9in, textwidth=5.5in, top=1in,
  headheight=12pt, headsep=25pt, footskip=30pt
}

\newpage
\appendix
\renewcommand{\thesection}{\Alph{section}}
\setcounter{section}{0}

\noindent{\LARGE\textbf{Appendix}\par}\normalsize

\vspace{10pt}
{\large \textbf{Contents}}
\startcontents[appendices]
\printcontents[appendices]{l}{1}{\setcounter{tocdepth}{3}}

\section{Engineering Axiom Set}
\label{app:axioms}

Our axiom set $\mathcal{A}$ contains 62 axioms: 41 anchored to a raw pool of 139 textbook-extracted candidates, plus 21 synthesized to fill coverage gaps for component families present in the CAD libraries but absent from the source corpus. The axioms follow the MECE principle: mutually exclusive on scope (no two axioms cover the same design decision) and collectively exhaustive over the CadQuery/cq\_warehouse/cq\_gears expressible space. This section describes the three-phase extraction pipeline and lists representative axioms.

\subsection{Axiom Extraction Pipeline}

The axiom set is produced through a three-phase LLM-assisted knowledge distillation pipeline that transforms multilingual assembly textbooks into a versioned, auditable set of engineering axioms.

\paragraph{Phase~1: Two-Stage LLM Extraction.}
The source corpus consists of MinerU-digitized assembly textbooks (Chinese and English). The corpus is recursively segmented into section-scoped chunks by heading level. Two sequential LLM stages then process each chunk:

\begin{itemize}[nosep,leftmargin=*]
\item \textbf{Stage~A (Recall-focused Knowledge Unit Extraction).}
For each chunk, an LLM extracts every atomic assembly-related claim as a \emph{knowledge unit} (KU), classified by claim type (topology, sequence, geometry, process, material specification, safety). The prompt explicitly prioritizes recall over precision: claims are extracted verbatim without filtering or generalization. This stage produced 210 chunks yielding the initial KU pool.

\item \textbf{Stage~B (Precision-focused Axiom Canonicalization).}
Each KU is individually assessed against strict axiom criteria: it must be (a)~topological or sequential in nature, (b)~component-class-scoped rather than product-specific, (c)~free of numerical specifications, and (d)~universally true rather than a recommendation. Accepted KUs are canonicalized into present-tense active voice, assigned to one of six categories (\texttt{fastener}, \texttt{gear}, \texttt{bearing}, \texttt{base\_frame}, \texttt{general\_sequence}, \texttt{general\_topology}), tagged with controlled-vocabulary scope labels, and given a confidence score. Rejected KUs are logged with structured rejection reasons (\texttt{not\_universal}, \texttt{numerical\_spec}, \texttt{not\_assembly}, \texttt{too\_vague}, \texttt{product\_specific}, \texttt{process\_only}).
\end{itemize}

The two-stage split enables localized failure analysis: a missing axiom is traceable to either a Stage-A recall miss or a Stage-B precision reject. The pipeline uses deterministic settings (seed=42, temperature=0.0) for reproducibility, and a schema validator enforces a controlled vocabulary of part types, feature types, relation types, and sequence constraints. This phase produced 139 raw candidate axioms.

\paragraph{Phase~2: LLM Review and Filtering.}
A second LLM pass reduces the raw pool through: (1)~clustering by deduplicated content key, (2)~per-cluster keep/merge/drop decisions, and (3)~cross-cluster semantic deduplication using embeddings. Six explicit exclusion rules are applied: axioms that merely restate geometric definitions, contain numerical specifications, are product-catalog-specific, are historical in nature, duplicate a sibling, or contradict a stronger sibling are dropped. Surviving axioms are reworded into canonical MUST/MUST NOT/SHOULD phrasing.

\paragraph{Phase~3: MECE Finalization.}
The final pass is an interactive audit enforcing four properties:

\begin{enumerate}[nosep,leftmargin=*]
\item \textbf{Mutual exclusivity}: every pair of axioms within each scope is checked for overlap.
\item \textbf{Universality}: each axiom is validated against demonstrator assembly examples.
\item \textbf{Coverage}: every port type and component class in the CAD extension library (Appendix~\ref{app:cad_assembly_ext}) must be covered by at least one axiom.
\item \textbf{DAG consistency}: precedes/follows relationships among axioms must form a directed acyclic graph.
\end{enumerate}

Where the textbook corpus lacked coverage for component families present in the CAD libraries (e.g., certain bearing configurations), \emph{synthesized axioms} are created with explicit provenance markers and engineering justification.

The disposition of the 139 raw axioms was: 37 kept directly, 16 merged (subsumed by existing axioms), 60 dropped as manufacturing-specific (out of scope), 19 dropped as overly narrow, and 7 dropped as trivial or tautological. An additional 4 axioms were broadened beyond their raw sources (counted as anchored but with synthesized generalizations), yielding 41 raw-anchored axioms. Finally, 21 axioms were synthesized to cover component families present in the CAD libraries but absent from the textbook corpus (e.g., washers, setscrews, keyways, helical/bevel/worm gears, sprockets, pulleys, couplings, plain bearings), each marked with explicit provenance and engineering justification. The final set contains $41 + 21 = 62$ axioms.

\subsection{Axiom Usage in the Pipeline}

Each axiom serves as a formal justification throughout the pipeline. In the Assembly Specification (Definition~\ref{def:spec}), every mate $m_j \in \mathcal{M}$ references one or more axioms via its $\mathcal{A}$ field. These references are: (1)~attached during semantic decomposition (Algorithm~\ref{alg:decompose}), (2)~used to guide component synthesis through provenance-aware prompts (Algorithm~\ref{alg:synthesis}), and (3)~propagated to the final verification report, creating an auditable chain from the output STEP file back through the assembly tree to the engineering principle and ultimately to the textbook source.

\subsection{Example Axioms}

Table~\ref{tab:example_axioms} lists representative axioms from the final set, illustrating the range of engineering knowledge encoded across different categories. Each axiom is stated in canonical MUST/MUST NOT phrasing and is scoped to a specific component class or assembly relation.

\begin{longtable}{cp{10cm}}
\caption{Representative engineering axioms from the final axiom set.}
\label{tab:example_axioms}\\
\toprule
\textbf{ID} & \textbf{Axiom Statement} \\
\midrule
\endhead
F-01 & An assembly is a set of parts related by pairwise connections; every connection between two parts removes one or more degrees of freedom from their relative motion. \\
F-03 & Every pairwise connection is expressed as one or more geometric constraints between named features (faces, edges, axes, or points) of the two parts. \\
C-02 & A coaxial constraint between two cylindrical axes makes them collinear, removing two translational and two rotational DOF and leaving one rotational and one translational DOF along the shared axis. \\
K-01 & A revolute joint constrains two parts to a single rotational DOF about a fixed axis. \\
G-02 & A spur-gear pair has parallel axes and cylindrical pitch surfaces; its teeth are straight and parallel to the gear axis. \\
T-01 & A threaded fastener creates a tensile preload along its shank that clamps the joined parts together in compression between the fastener's bearing surfaces. \\
\bottomrule
\end{longtable}

\section{CAD Assembly Extension Library}
\label{app:cad_assembly_ext}

The assembly realization pipeline described in Section~\ref{sec:method} relies on a custom CadQuery extension library, \texttt{cad-assembly-ext}, that introduces a \emph{port-based mating system} for deterministic mechanical assembly. This appendix details its architecture, core abstractions, and the design decisions motivated by the requirements of LLM-driven CAD generation.

\subsection{Motivation}

Vanilla CadQuery provides only raw \texttt{cq.Location}-based placement for multi-part assemblies, offering no semantic mating, no type safety, and no validation layer. This is insufficient for an LLM-driven pipeline for three reasons: (1)~LLMs frequently produce structurally plausible but geometrically incorrect assemblies (e.g., bolts mated to gear meshes), requiring compile-time type checking on assembly operations; (2)~boolean operations in the OpenCascade kernel silently produce disconnected compounds when solids share only boundary contact, a failure mode that LLM-generated code encounters systematically; and (3)~LLM-synthesized components may contain \emph{phantom features}---ports declared at locations where the corresponding geometric operation (e.g., a bore cut) silently failed---requiring geometric evidence verification beyond symbolic annotations.

\subsection{Architecture}

The library follows a three-phase data flow that mirrors the assembly realization pipeline:

\paragraph{Modeling Phase.}
Parts are constructed using standard CadQuery operations augmented with nine extension methods (e.g., \texttt{cbore\_hole\_with\_port}, \texttt{shaft\_port}, \texttt{bore\_port}) that atomically create geometry \emph{and} register a typed port annotation in a single call. This co-location prevents geometry--metadata drift.

\paragraph{Assembly Phase.}
\texttt{AssemblyExt.add()} extracts ports from each part's context. The \texttt{.mate()} method resolves port references (string format \texttt{"part\_name:port\_name"}), checks type compatibility via a predefined compatibility matrix, and dispatches to one of seven mate-type functions that each compute a closed-form rigid-body transform.

\paragraph{Export Phase.}
\texttt{.build()} produces a standard \texttt{cq.Assembly} object; \texttt{.export\_step()} writes STEP files with preserved assembly hierarchy.

\subsection{Port Abstraction}

The port abstraction is the central data structure that enables typed, semantically meaningful assembly interfaces. Each port encodes not only a spatial coordinate frame but also a type label and type-specific parameters, allowing the assembly engine to enforce compatibility constraints before geometric computation.

A \emph{port} is a typed geometric attachment point:
\[
\pi = (\textit{name},\; t \in \mathcal{T}_{\mathrm{port}},\; L \in SE(3),\; \phi),
\]
where $t$ is one of 12 port types (\texttt{BOLT\_SEAT}, \texttt{SHAFT}, \texttt{BORE}, \texttt{GEAR\_MESH}, \texttt{THREAD\_INTERNAL}, \texttt{THREAD\_EXTERNAL}, \texttt{PRESS\_FIT}, \texttt{FLAT\_FACE}, \texttt{CLEARANCE\_HOLE}, \texttt{CBORE\_SEAT}, \texttt{HOLE}, \texttt{SNAP\_FIT}), $L$ is the port's coordinate frame, and $\phi$ contains type-specific parameters (e.g., diameter, depth, thread specification).

A compatibility matrix $\mathcal{C} \subseteq \mathcal{T}_{\mathrm{port}} \times \mathcal{T}_{\mathrm{port}}$ defines which port-type pairs may be mated, providing semantic guard rails that catch category errors (e.g., attempting to mate a fastener with a gear mesh) before geometric computation.

\subsection{Deterministic Mate Transform}

The central mathematical operation is a closed-form rigid-body transform, as given in Eq.~(\ref{eq:mate_transform}) of the main text. Given a base port $\pi_b$ (already placed in world coordinates) and an incoming port $\pi_c$ (in the incoming part's local frame), the transform that places the incoming part is
$T = L_b \cdot R_{\mathrm{flip}} \cdot R_\alpha \cdot L_c^{-1}$,
where $R_{\mathrm{flip}}$ is a $180^\circ$ rotation about the local $x$-axis (making $z$-axes anti-parallel for face-to-face contact), and $R_\alpha$ is an optional user-specified rotation about the shared $z$-axis. This single matrix multiplication---with no iterative constraint solving---produces bit-identical results across runs, satisfying Proposition~\ref{prop:determinism}.

\subsection{Mate Types}

The library implements seven mate types, each extending the base transform (Eq.~\ref{eq:mate_transform}) with type-specific geometric semantics and parameter validation. Table~\ref{tab:mate_impl} summarizes their implementations.

\begin{table}[H]
\centering
\small
\caption{Mate type implementations with their geometric semantics.}
\label{tab:mate_impl}
\begin{tabular}{@{}lp{7.5cm}@{}}
\toprule
\textbf{Mate Type} & \textbf{Semantics} \\
\midrule
\texttt{face\_to\_face} & Coplanar normals with optional in-plane offsets $(u,v)$ and normal displacement $d_n$. \\
\texttt{coaxial} & Anti-parallel $z$-axes with optional depth offset along the shared axis. \\
\texttt{coaxial\_face} & Axis alignment with contact; used for bolt-in-counterbore placement. \\
\texttt{gear\_mesh} & Validates matching module $m$ and pressure angle; computes center distance $(z_1+z_2)\cdot m/2$ and applies phase correction when $(z_1+z_2)$ is odd. \\
\texttt{press\_fit} & Validates shaft diameter $\geq$ hole diameter (interference condition), then delegates to coaxial. \\
\texttt{thread\_engage} & Validates thread specification compatibility, then delegates to coaxial with depth. \\
\texttt{snap\_to\_face} & Parametric $(u,v)$ positioning on a planar face. \\
\bottomrule
\end{tabular}
\end{table}

\subsection{Robust Boolean Operations}

To address the systematic failure mode where CadQuery's \texttt{union()} silently produces disconnected compounds under face-only contact, the library provides \textsc{SafeUnion} (Proposition~\ref{prop:safeunion}), which employs a three-strategy cascade: (1)~standard union when volumetric overlap exists; (2)~glue-based fusion when bounding boxes touch; and (3)~automatic micro-extension along the shortest gap direction when neither applies. A post-condition check asserts exactly one output solid, raising an error with diagnostic bounding-box information rather than propagating a disconnected compound.

\subsection{Port-Geometry Verification}

The library verifies that declared ports correspond to actual geometric features using OpenCascade's \texttt{BRepClass3d\_SolidClassifier}, implementing the port-geometry consistency checks defined in Definition~\ref{def:consistency}:

\begin{itemize}[nosep,leftmargin=*]
\item \textbf{Bore check}: samples axial points along the declared bore direction (must classify as \texttt{OUT} indicating cavity) and radial ring points (at least one must classify as \texttt{IN} indicating surrounding material).
\item \textbf{Shaft check}: samples axial points at quarter and three-quarter shaft length (must classify as \texttt{IN}) and a radial point outside the shaft diameter (must classify as \texttt{OUT}).
\item \textbf{Flat face check}: iterates all B-Rep faces on the solid to find one with collinear normal (within $10^{-2}$\,rad) passing within 1\,mm of the port origin.
\end{itemize}

Verification failures produce structured error messages with natural-language hints designed for LLM-consumable repair feedback, directly supporting the iterative generation-repair loop in Algorithm~\ref{alg:synthesis}.

\subsection{Factory and Adapter System}

The library includes 8 parametric base-frame factories (Table~\ref{tab:factory_registry}) and adapter modules for two third-party libraries, covering fasteners, bearings, and gears:

\begin{itemize}[nosep,leftmargin=*]
\item \textbf{Fastener adapters}: wrap \texttt{cq\_warehouse} screws, nuts, and washers with 5, 3, and 3 ports respectively, including thread ports, head seats, and bearing faces.
\item \textbf{Gear adapters}: wrap \texttt{cq\_gears} spur gears with bore, mesh point (at pitch radius with transmission parameters in $\phi$), and face ports.
\item \textbf{Bearing adapters}: wrap \texttt{cq\_warehouse} bearings with bore, outer race (\texttt{PRESS\_FIT} type), and face ports.
\end{itemize}

A \texttt{user.*} factory namespace enables LLM-synthesized custom factories via dynamic import, supporting the open-world component synthesis described in Section~\ref{sec:method}.

\subsection{Declarative Tree-Driven Assembly}

The function \texttt{assemble\_from\_tree()} accepts a JSON specification containing \texttt{parts\_index} (ordered dictionary of part nodes) and \texttt{mates} (ordered list of mate declarations). It materializes all parts via factory dispatch, adds the first part as the base, then iterates mates in declaration order---each calling \texttt{AssemblyExt.mate()}. This declarative interface maps naturally to the structured Assembly Specification (Definition~\ref{def:spec}), enabling the LLM to output structured data rather than imperative CadQuery code.

\section{Built-in Parametric Factory Registry}
\label{app:factories}

The assembly library ships with 13 built-in parametric factories that cover the most commonly occurring mechanical components in the benchmark assemblies. Table~\ref{tab:factory_registry} summarizes each factory together with its key parameters and the typed ports it declares. These factories serve as the default component source during assembly realization; components not covered by the registry are synthesized on demand via the LLM (Algorithm~\ref{alg:synthesis}).

\begin{table}[H]
\centering
\small
\caption{The 13 built-in parametric factories with their key parameters and declared ports.}
\label{tab:factory_registry}
\begin{tabular}{llp{5cm}}
\toprule
\textbf{Factory} & \textbf{Key Parameters} & \textbf{Ports} \\
\midrule
\texttt{rect\_plate} & $L, W, T$ & top\_face, bottom\_face, 4$\times$ side\_face \\
\texttt{l\_bracket} & $L_1, L_2, W, T$ & face\_a, face\_b, web\_face \\
\texttt{u\_channel} & $L, W, H, T$ & floor\_face, 2$\times$ wall\_face \\
\texttt{boss} & $D, H, \text{bore}$ & top\_face, bore\_port \\
\texttt{bolt\_circle\_flange} & $D, H, \text{bolt\_d}, n$ & top\_face, bottom\_face, bolt\_circle \\
\texttt{housing\_block} & $L, W, H, \text{bore}$ & 6 faces, bore\_port \\
\texttt{stepped\_shaft} & segments & per-segment shaft\_seats \\
\texttt{coupling\_hub} & $D, L, \text{bore}$ & shaft\_seat, face \\
\texttt{fastener.bolt} & $d, L, \text{head}$ & thread, head\_face \\
\texttt{fastener.nut} & $d, \text{width}$ & thread\_internal, face \\
\texttt{fastener.washer} & $d_i, d_o, T$ & top\_face, bottom\_face \\
\texttt{gear.spur} & $m, z, \text{width}, \text{bore}$ & mesh\_point, bore \\
\texttt{bearing.ball} & $d_i, d_o, W$ & inner\_ring, outer\_ring \\
\bottomrule
\end{tabular}
\end{table}

\section{YAML Specification Format}
\label{app:spec_format}

The Assembly Specification (Definition~\ref{def:spec}) is serialized as a YAML document that the LLM produces during semantic decomposition. The format declares parts with their factory references and parameters, followed by mates that specify port endpoints, mate types, and axiom references. The following listing shows a representative specification for a single-stage spur gear reducer.

\begin{verbatim}
assembly:
  id: gear_reducer_v1
  description: "Single-stage spur gear reducer"
  parts:
    - id: housing
      factory: base_frame.housing_block
      params: {length: 120, width: 80, height: 60, bore_d: 25}
    - id: input_gear
      factory: gear.spur
      params: {module: 2, teeth: 20, width: 15, bore_d: 12}
    - id: output_gear
      factory: gear.spur
      params: {module: 2, teeth: 40, width: 15, bore_d: 16}
  mates:
    - id: m_input_shaft
      mate_type: coaxial
      base: {ref: housing, port: bore_port}
      incoming: {ref: input_gear, port: bore}
      axioms: [C-02, P-01]
    - id: m_gear_mesh
      mate_type: gear_mesh
      base: {ref: input_gear, port: mesh_point}
      incoming: {ref: output_gear, port: mesh_point}
      axioms: [G-02]
\end{verbatim}

\section{Failure Mode Catalog}
\label{app:failures}

During development and evaluation, we systematically cataloged recurring failure modes encountered in LLM-driven CAD assembly generation. Table~\ref{tab:failure_modes} summarizes the five most representative failure modes, their root causes, observable symptoms, and the specific validation gate in \textsc{AssemCAD} that detects and prevents each failure from propagating to the final assembly.

\begin{table}[H]
\centering
\small
\caption{Systematic catalog of identified failure modes, their root causes, and the system component that addresses each.}
\label{tab:failure_modes}
\begin{tabularx}{\textwidth}{lXXl}
\toprule
\textbf{Failure Mode} & \textbf{Root Cause} & \textbf{Symptom} & \textbf{Gate} \\
\midrule
Phantom bore & \texttt{Workplane("XZ")} sign trap causes \texttt{.cut()} to miss solid & Axle transfixes un-drilled arm ($\pi r^2 L$ clash) & Port-geom. \\
Disjoint compound & \texttt{union(glue=False)} on face-contact & Part renders as floating sub-pieces & \texttt{safe\_union} \\
All-coincident parts & Zero-offset \texttt{face\_to\_face} to same port & Multiple parts at same location & Overlap validator \\
Port-frame surprise & Auto-derived $u$/$v$ axes on oblique normal & Offset in unexpected world direction & Strict \texttt{x\_dir} \\
Semantic mis-selection & LLM picks wrong factory for role & Geometry correct but functionally wrong & Visual reflection \\
\bottomrule
\end{tabularx}
\end{table}

\section{Benchmark Construction}
\label{sec:benchmark_construction}

\vspace{2em}
\begin{wraptable}{r}{0.35\linewidth}
\centering
\caption{Statistics of AssemBench.}
\label{tab:benchmark_stats}
\small
\setlength{\tabcolsep}{5pt}
\begin{tabular}{lr}
\toprule
\textbf{Metric} & \textbf{Value} \\
\midrule
Assemblies & 120 \\
Rendered Image sets & 120 \\
GT STEP Models & 120 \\
\midrule
Avg. Bodies & 3.38 \\
Avg. Occurrences & 4.52 \\
Avg. Joints & 2.23 \\
Avg. Contacts & 9.43 \\
Avg. Holes & 4.78 \\
Avg. B-rep Faces & 174.81 \\
\midrule
Avg. Prompt Length & 214 words \\
Free-form Ratio & 0.9\% \\
\bottomrule
\end{tabular}
\vspace{-4mm}
\end{wraptable}

\subsection{Construction Pipeline}

\textsc{AssemBench} is constructed upon the assembly subset of the Fusion360 Gallery dataset~\cite{willis2021fusion360}. Rather than manually annotating CAD assemblies, we develop a semi-automatic annotation pipeline that combines rule-based parsing with vision-language understanding to produce standardized natural-language assembly specifications.

Given an assembly instance, we first parse its native Fusion360 assembly JSON file using a rule-based parser. The parser extracts structured assembly information, including component hierarchy, assembly occurrences, joints, contacts, and other assembly metadata. This structured representation provides an explicit description of the assembly topology and engineering relationships.

To complement the structural information, each assembly is rendered from two canonical viewpoints using the original CAD geometry. The rendered images preserve the global appearance and spatial layout of the assembly, providing visual cues that are difficult to infer solely from symbolic assembly graphs.

Finally, both the parsed assembly structure and the rendered multi-view images are provided to a vision-language model (VLM). Rather than performing generic image captioning, the VLM acts as an engineering-aware annotator that jointly reasons over geometric appearance and assembly structure to generate standardized CAD briefs. During annotation, the VLM is instructed to preserve engineering semantics, accurately describe functional components and assembly relationships, reflect structural complexity, and avoid unnecessary implementation details. This process produces natural-language specifications that remain faithful to the underlying CAD assemblies while maintaining consistent annotation quality across the benchmark.

Figure~\ref{fig:benchmark_pipeline} illustrates the complete construction pipeline.

\begin{figure*}[t]
\centering
\includegraphics[width=\textwidth]{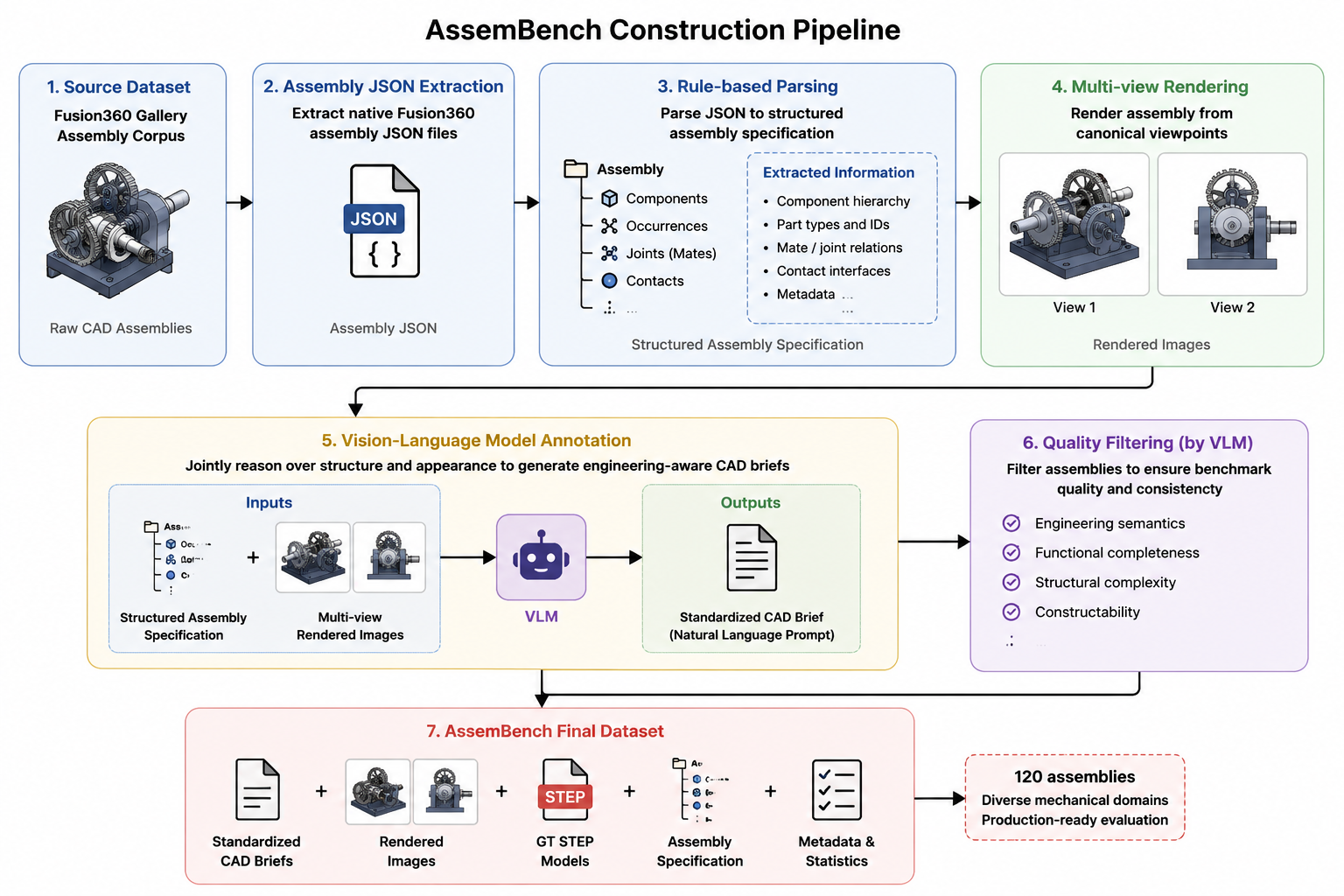}
\caption{
The pipeline of constructing AssemBench
}
\label{fig:benchmark_pipeline}
\end{figure*}

\subsection{Prompt Generation}

The generated CAD briefs are intended to describe \emph{what} should be constructed rather than \emph{how} to implement the corresponding CAD program. Therefore, the prompts emphasize engineering functionality, component interactions, and assembly semantics instead of low-level geometric operations or CAD-specific commands.

To improve consistency, all prompts follow a unified annotation protocol. In particular, the VLM is encouraged to:

\begin{itemize}[leftmargin=*]
    \item describe the overall mechanical function of the assembly;
    \item identify major components and their functional roles;
    \item explain key assembly relationships and engineering semantics;
    \item preserve the structural complexity of the original assembly;
    \item avoid implementation-specific details such as CAD operations or programming instructions.
\end{itemize}

This standardized annotation protocol substantially reduces stylistic variation across prompts while preserving the engineering information required for production-ready CAD assembly generation. The quantitative results are shown in Table~\ref{tab:benchmark_stats}

\section{Benchmark Visualization}
\label{sec:benchmark_visualization}

Figure~\ref{fig:benchmark_examples} shows representative examples from \textbf{AssemBench}. Each example contains decomposed CAD components and the corresponding assembled mechanical system.

\begin{figure*}[t]
    \centering
    \setlength{\tabcolsep}{2pt}
    \renewcommand{\arraystretch}{0.9}

    \begin{tabular}{cccccc}
        \multicolumn{5}{c}{\textbf{Components}} & \textbf{Assembly} \\[-0.5mm]

        \includegraphics[width=0.095\linewidth]{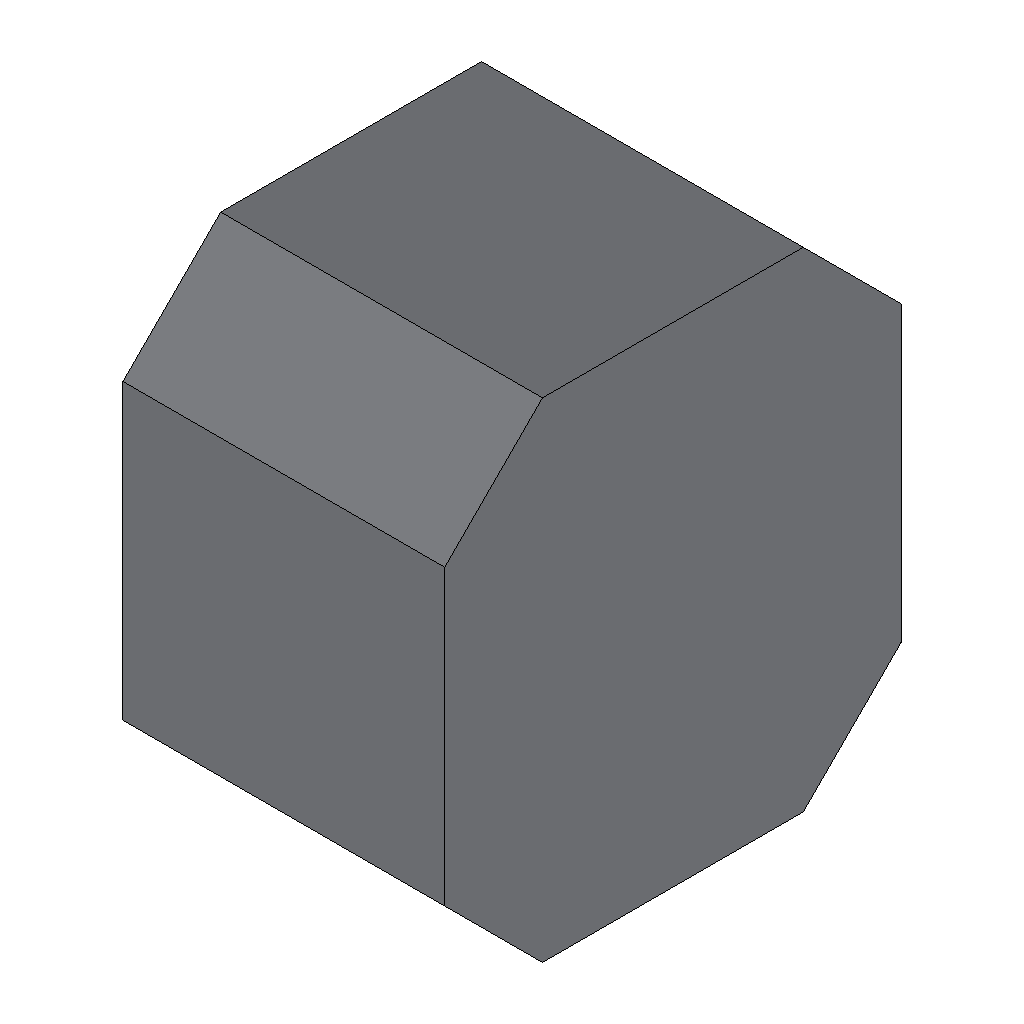} &
        \includegraphics[width=0.095\linewidth]{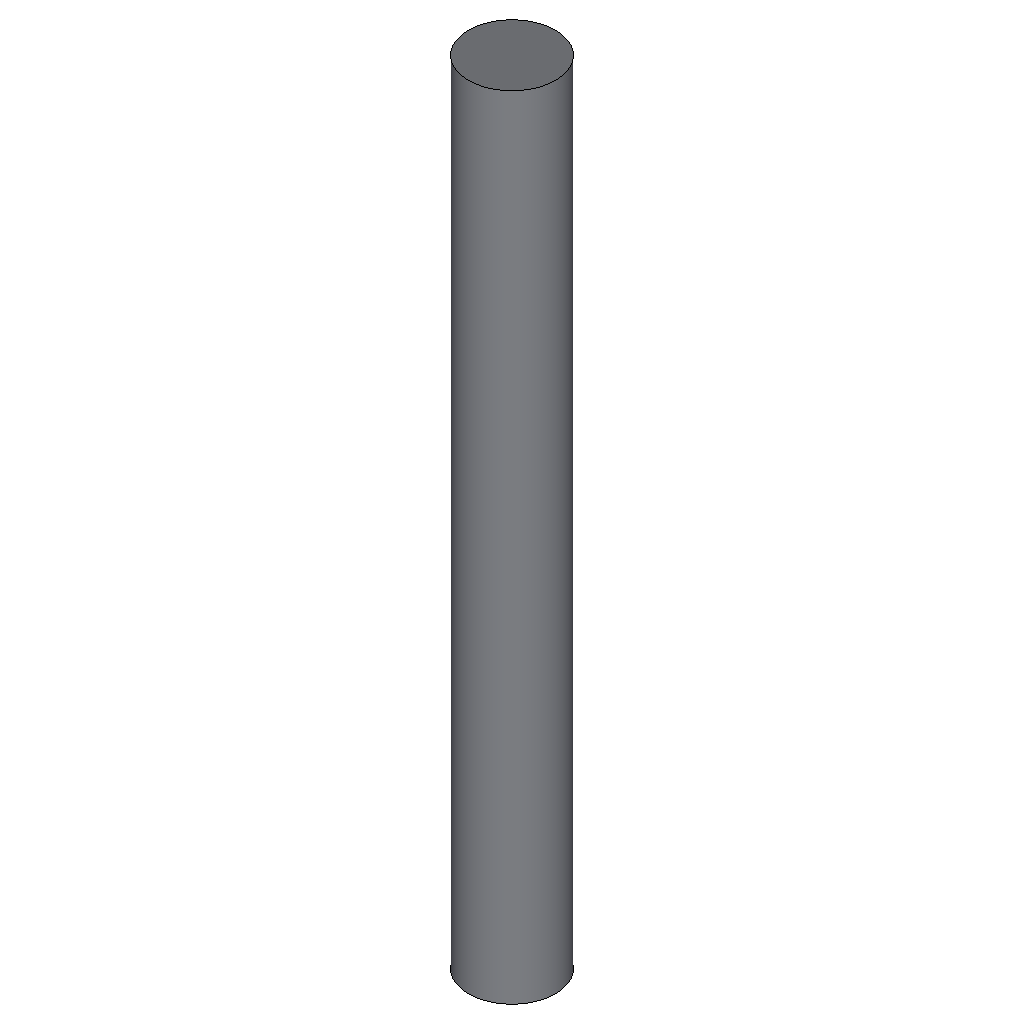} &
        \includegraphics[width=0.095\linewidth]{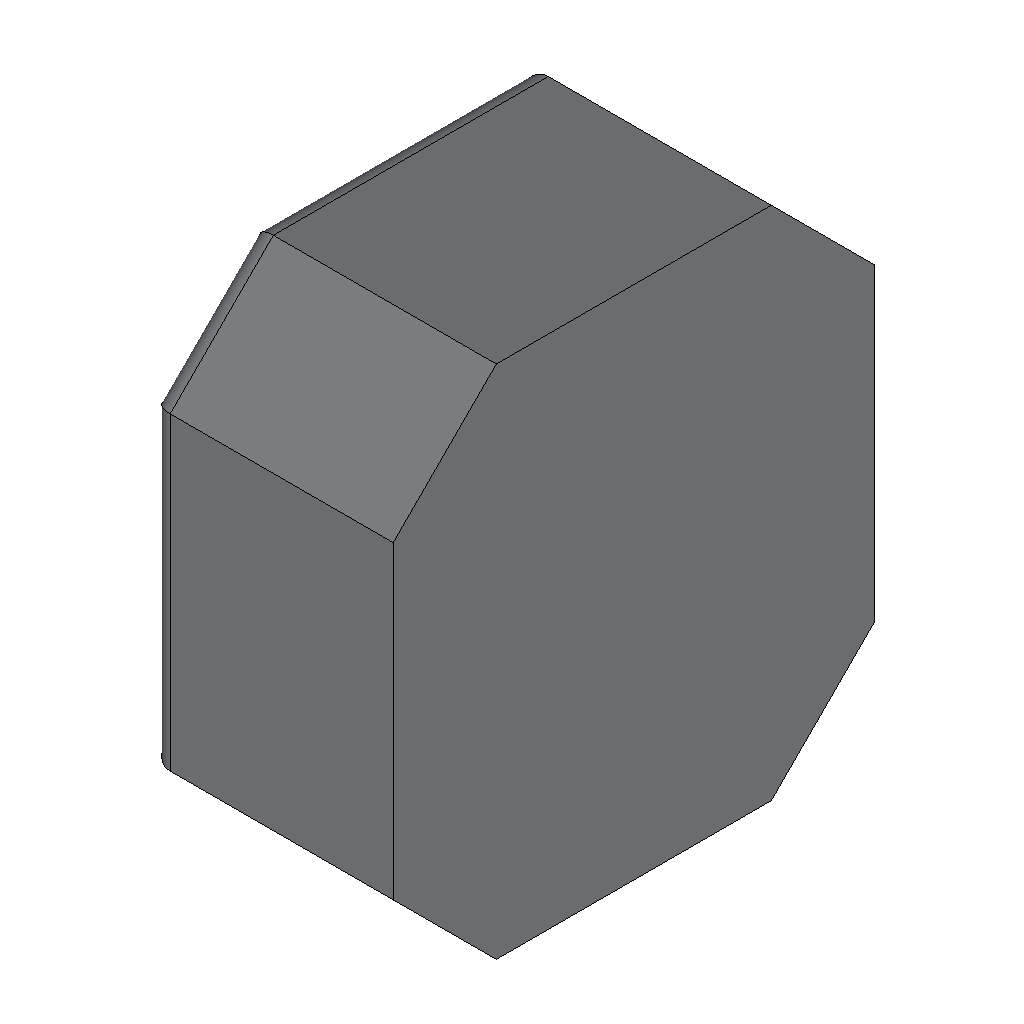} &
        \includegraphics[width=0.095\linewidth]{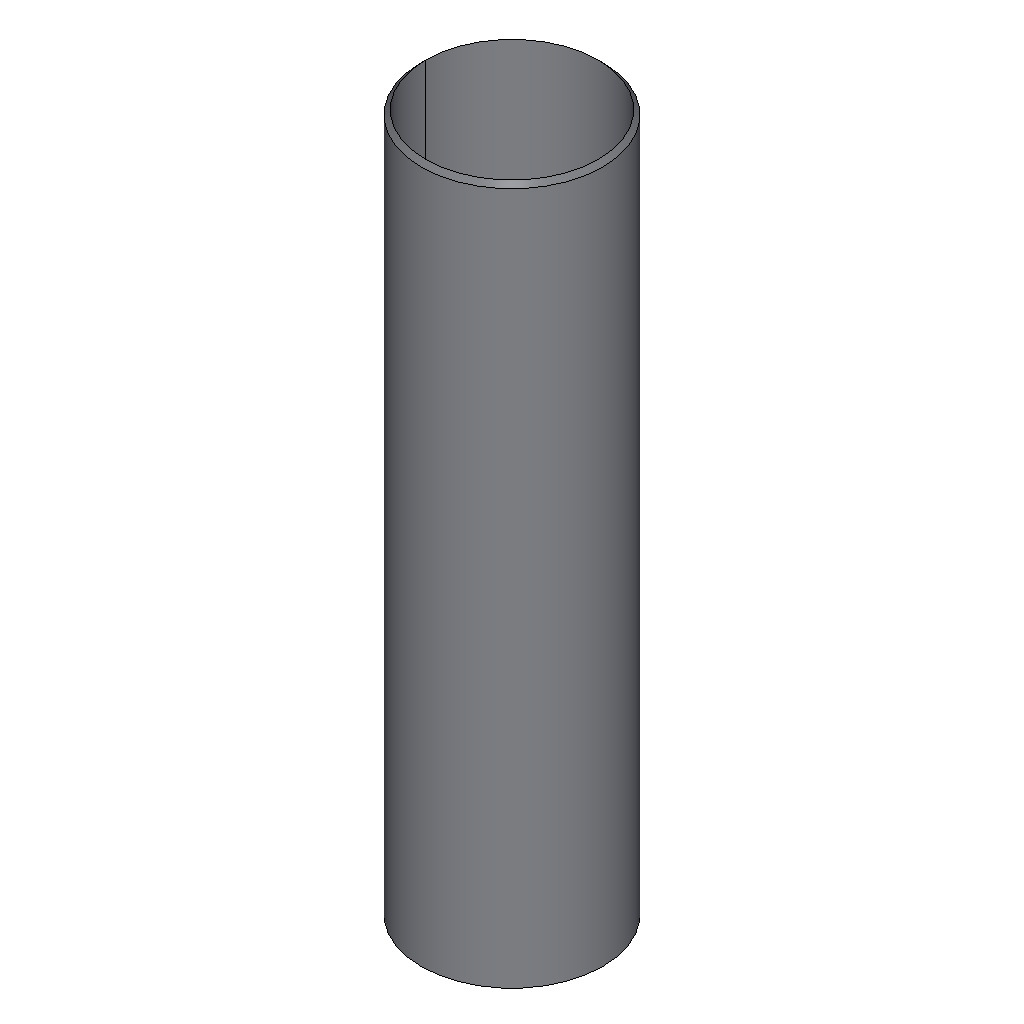} &
        \includegraphics[width=0.095\linewidth]{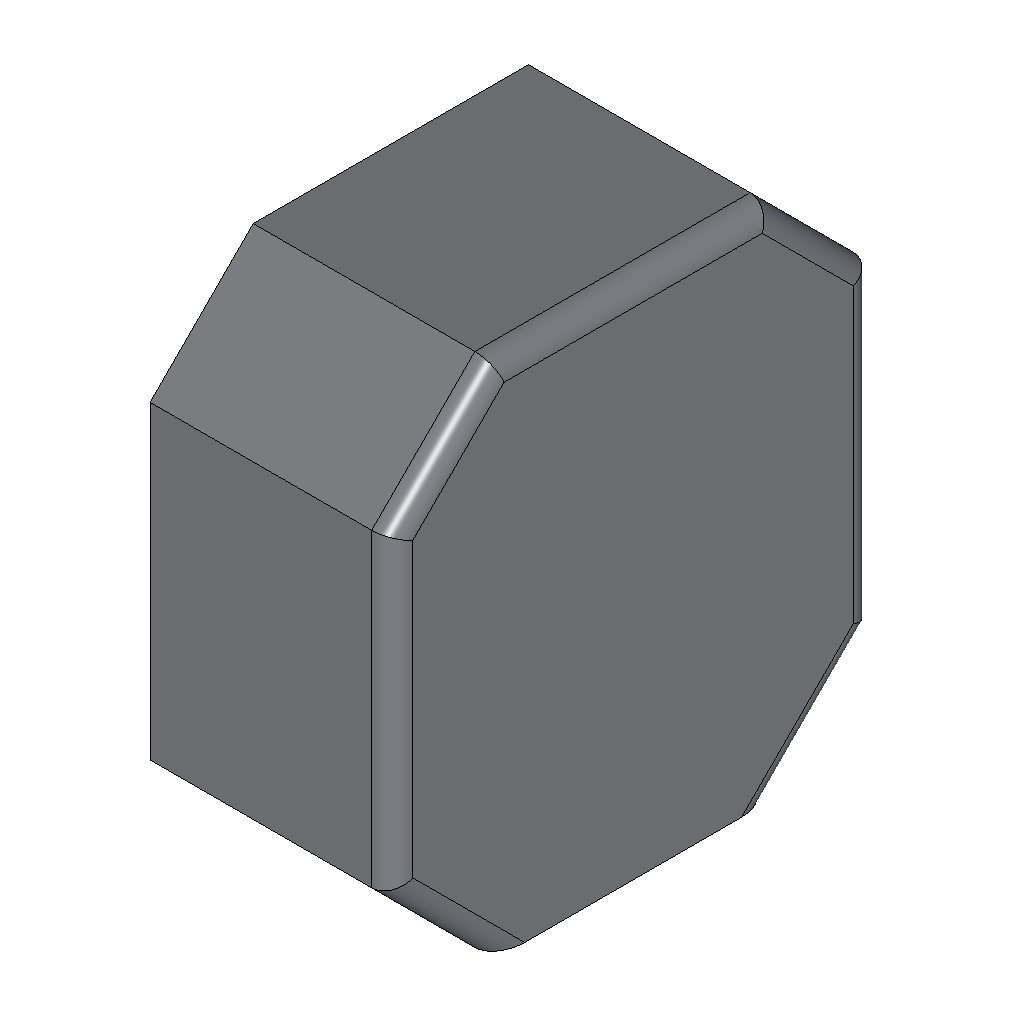} &
        \includegraphics[width=0.19\linewidth]{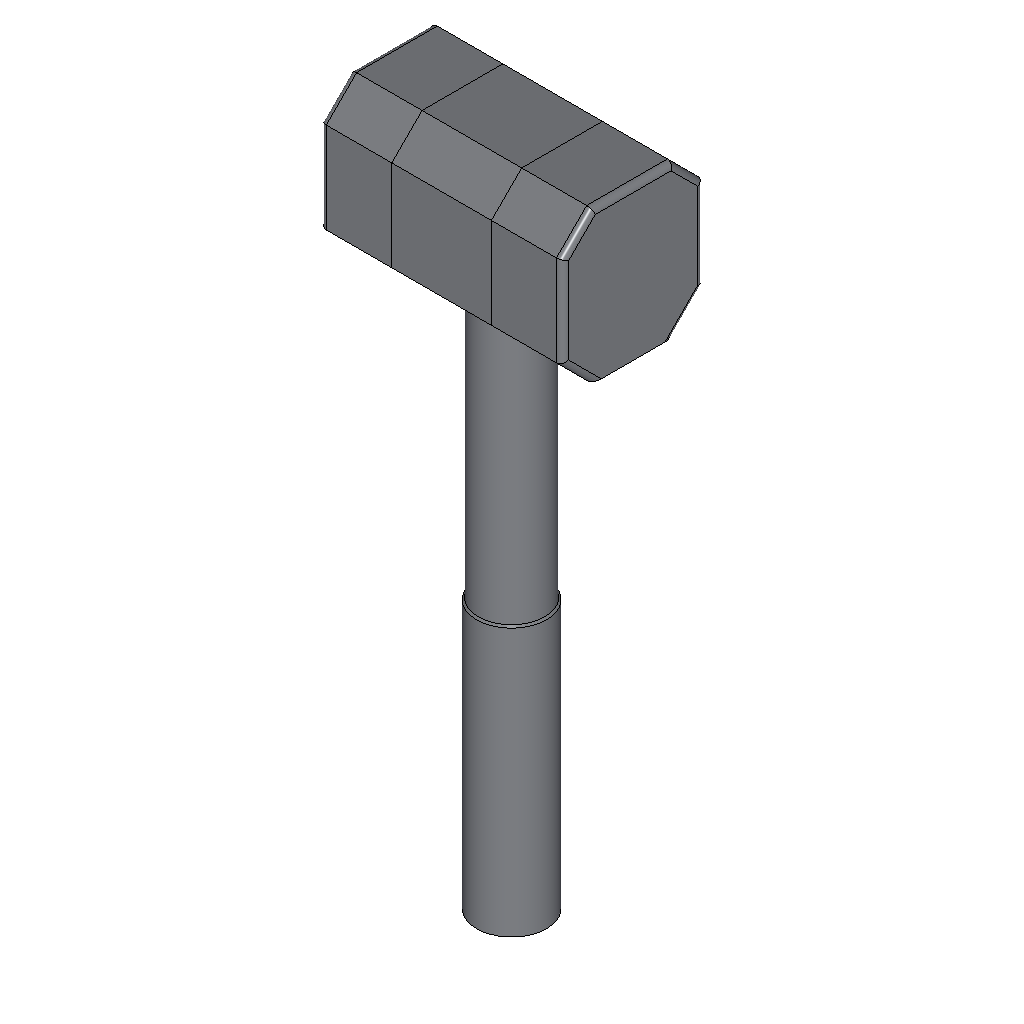}
        \\[2mm]

        \includegraphics[width=0.095\linewidth]{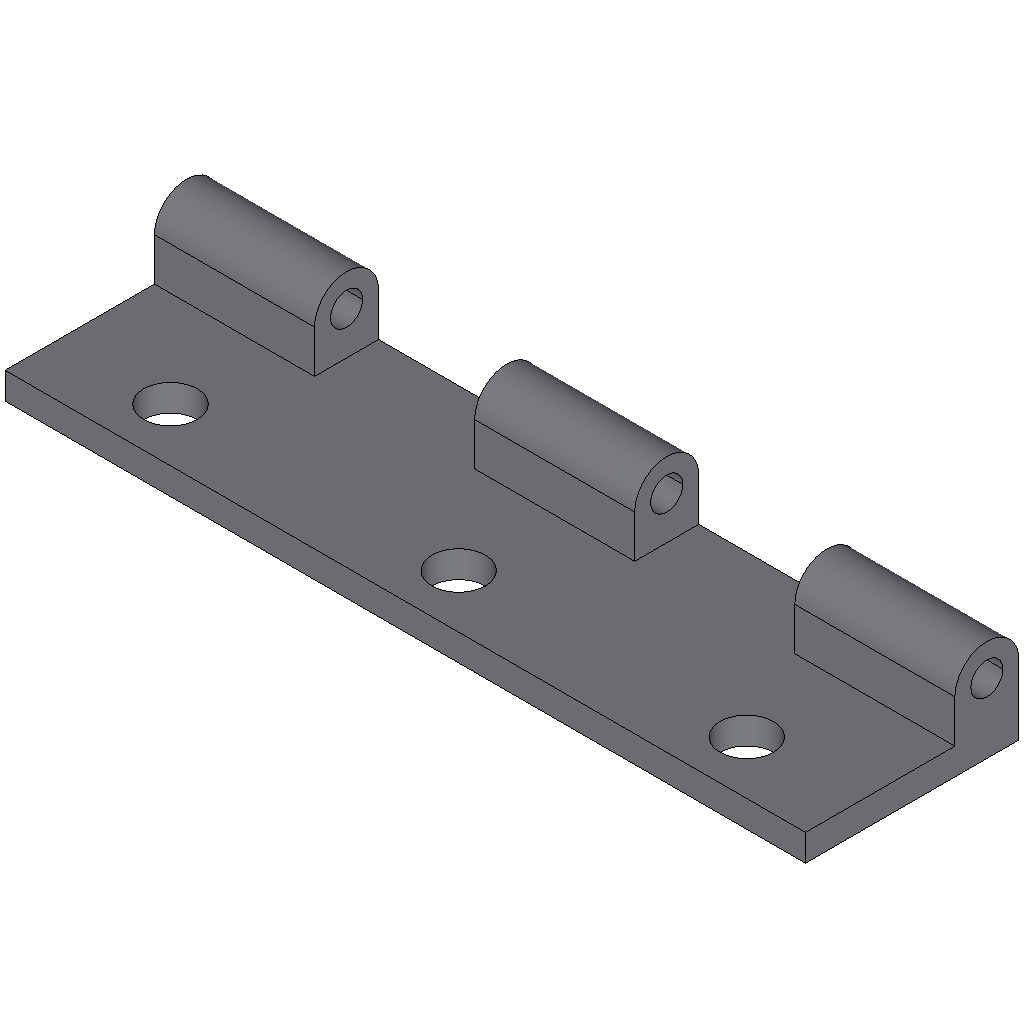} &
        \includegraphics[width=0.095\linewidth]{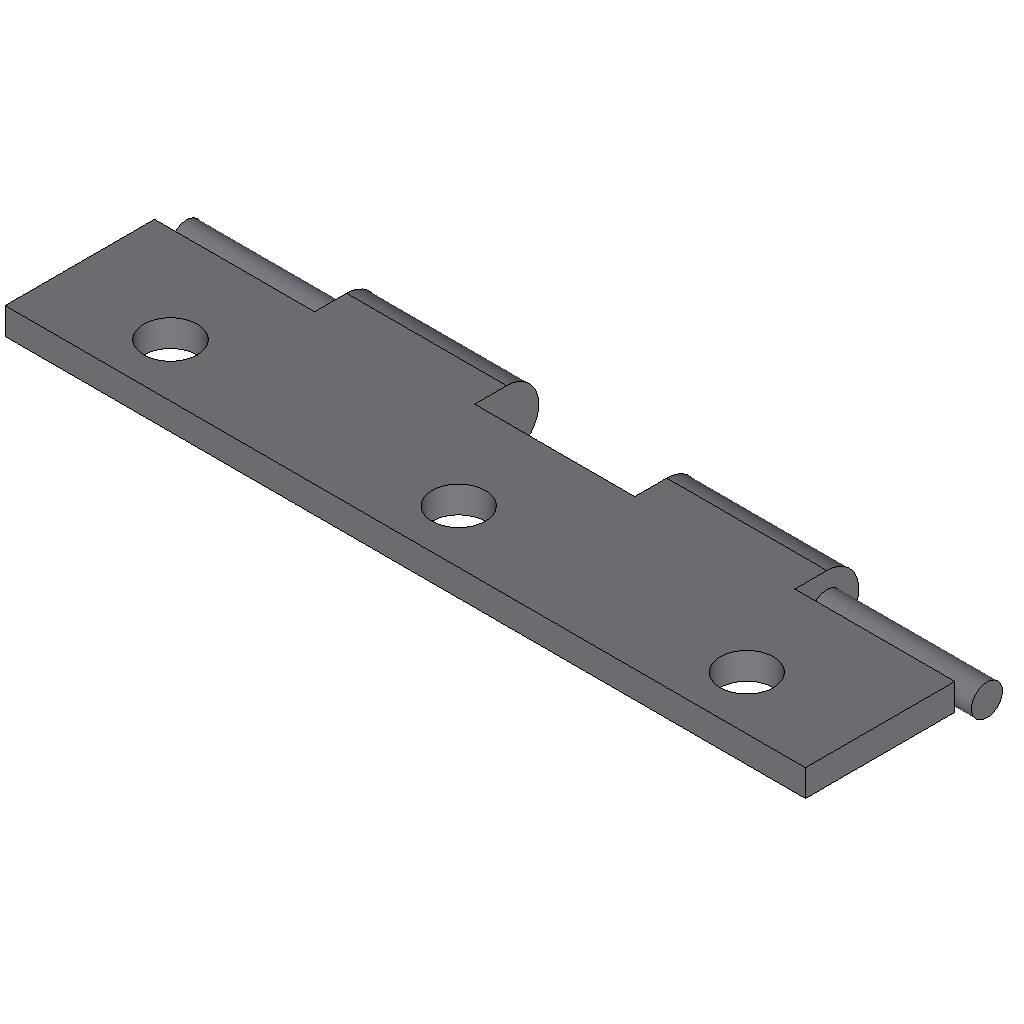} &
        & & &
        \includegraphics[width=0.19\linewidth]{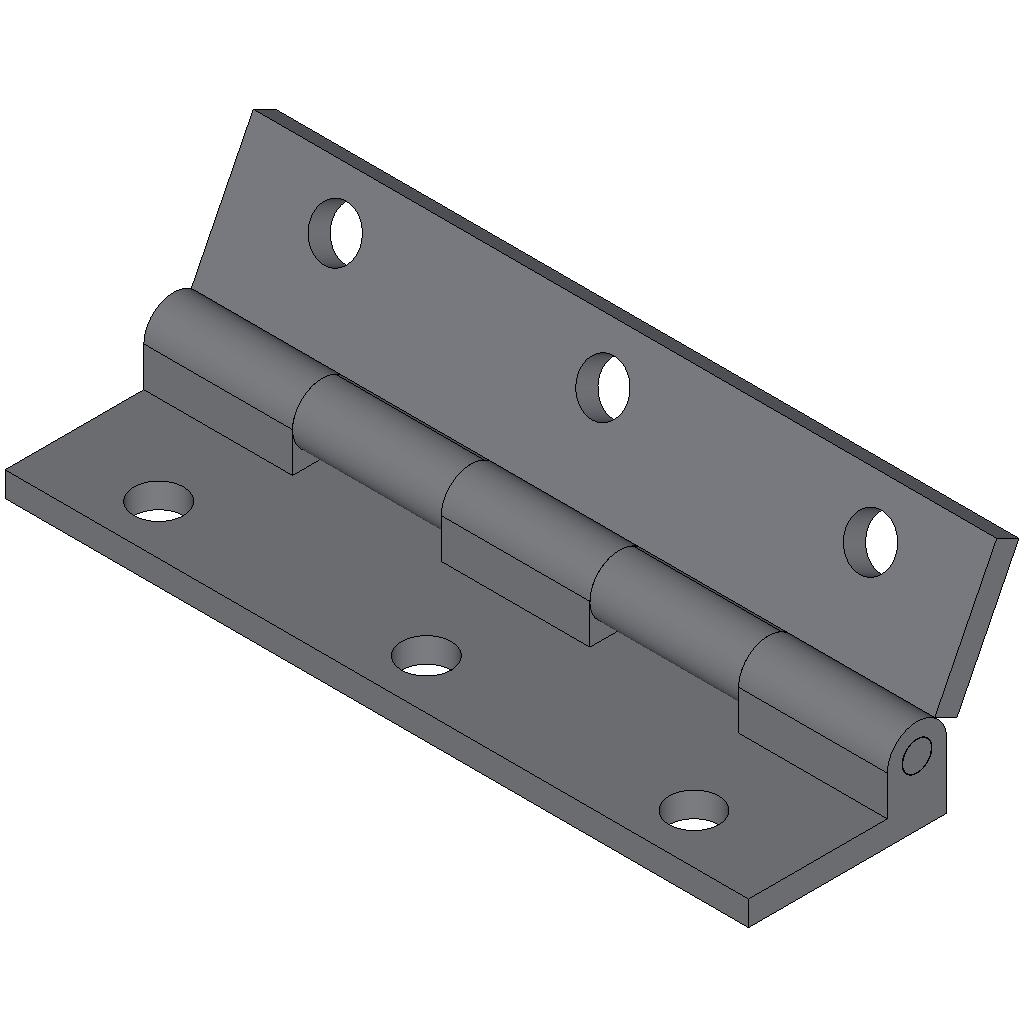}
        \\[2mm]
        \includegraphics[width=0.095\linewidth]{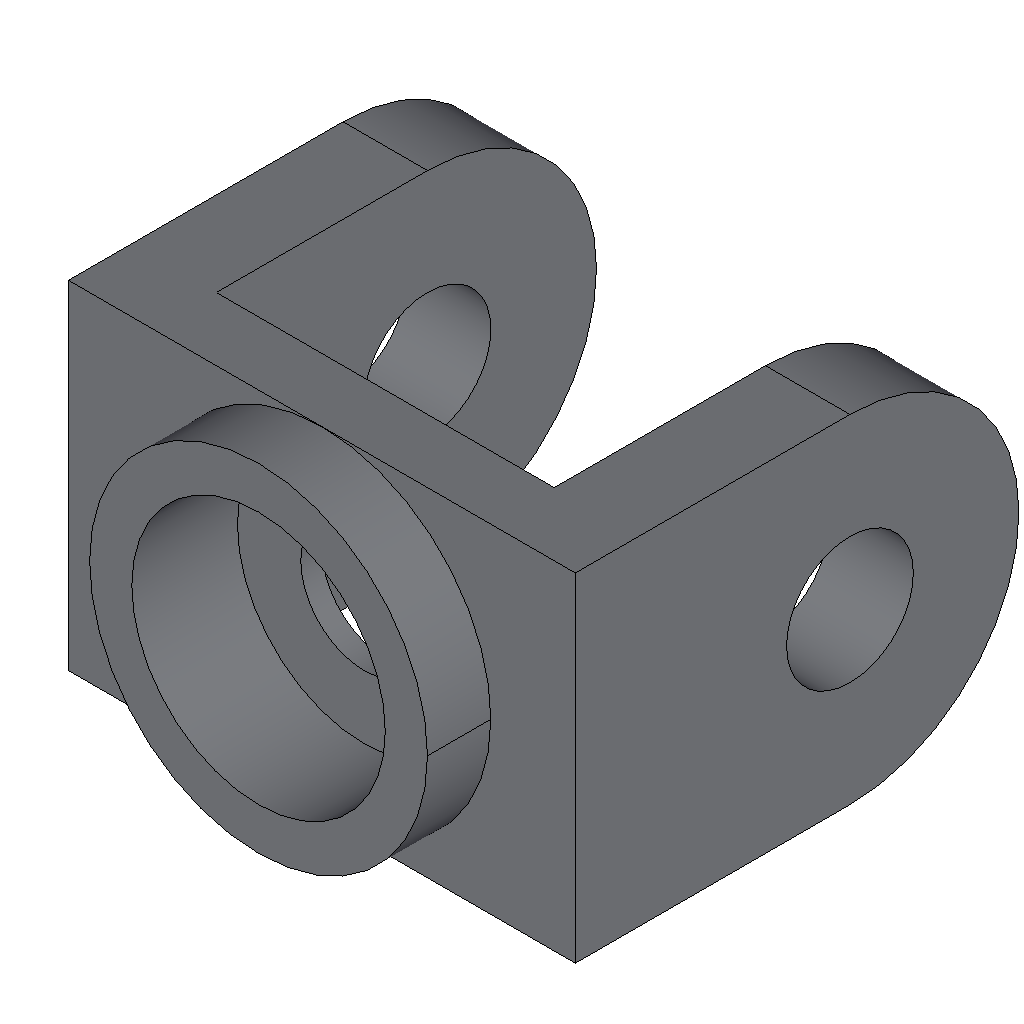} &
        \includegraphics[width=0.095\linewidth]{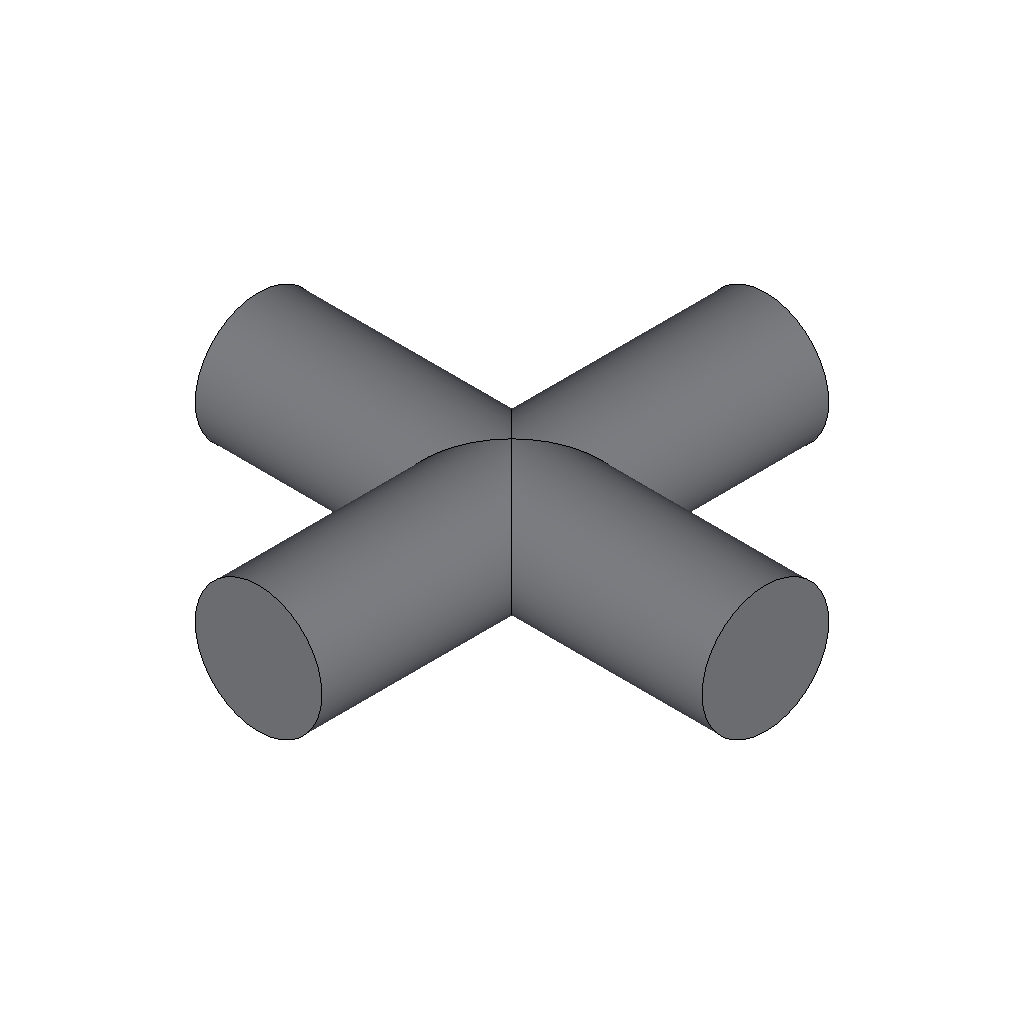} &
        \includegraphics[width=0.095\linewidth]{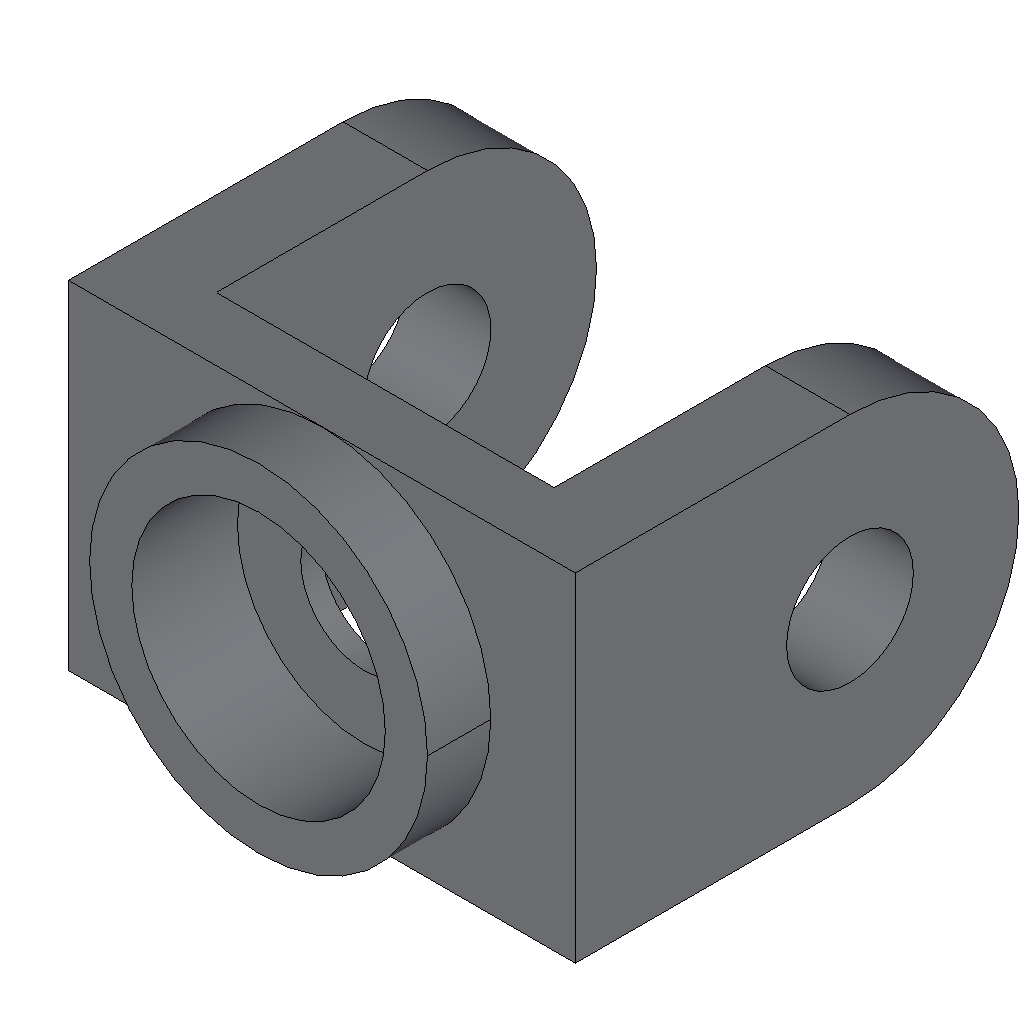} &
        & &
        \includegraphics[width=0.19\linewidth]{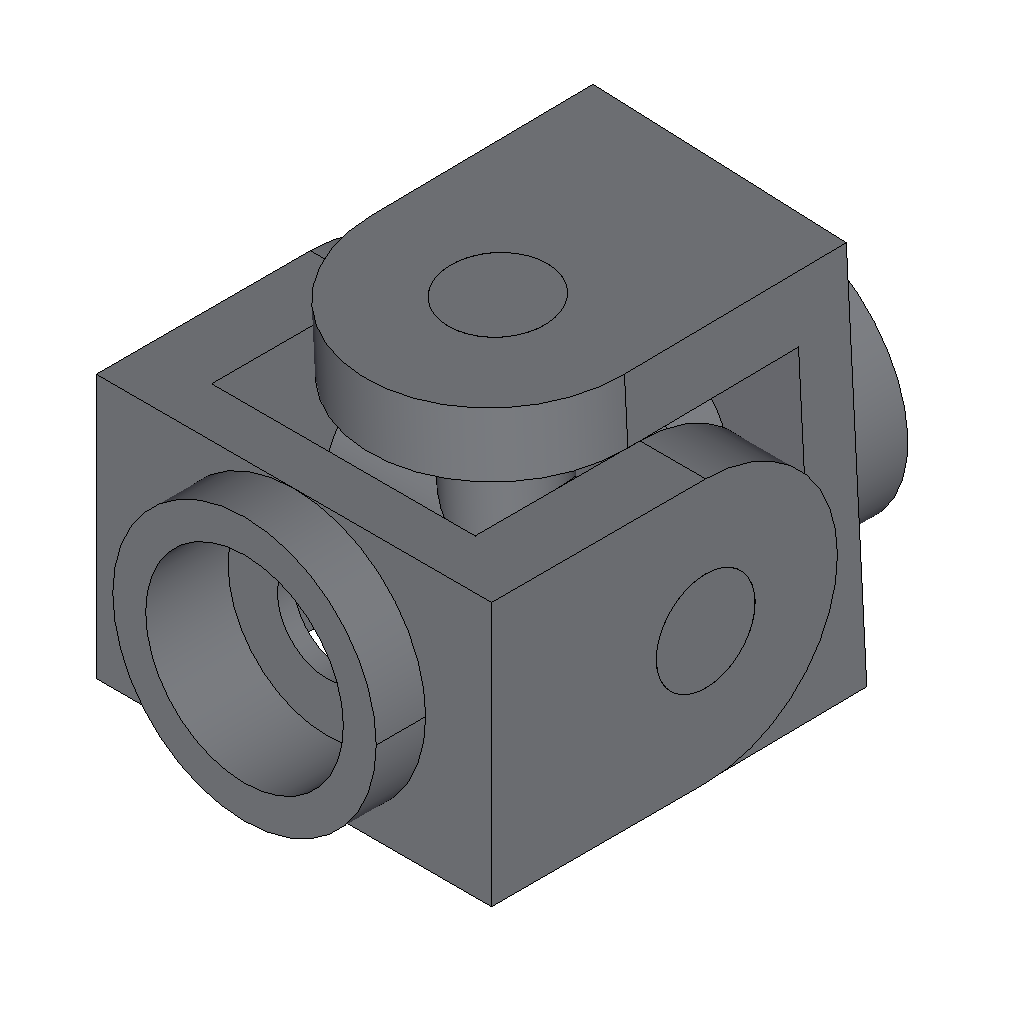}
    \end{tabular}

    \vspace{-1mm}
    \caption{
    Representative examples from \textbf{AssemBench}. Each row shows decomposed CAD components and the corresponding assembled mechanical system.
    }
    \label{fig:benchmark_examples}
\end{figure*}

\section{On the Unreliability of Chamfer Distance for Assembly Evaluation}
\label{app:cd_unreliability}

Chamfer Distance (CD) is widely used to evaluate geometric similarity in 3D reconstruction and CAD generation tasks. However, we argue that CD is fundamentally unreliable for evaluating complex mechanical assemblies. This section provides a formal analysis demonstrating two critical deficiencies: CD lacks \emph{assembly structure identifiability} and \emph{assembly ranking consistency}.

\subsection{Background and Definition}

We first recall the standard definition of Chamfer Distance and describe how it is applied to assembly evaluation.

\begin{definition}[Chamfer Distance]
Given two finite point sets $P=\{p_i\}_{i=1}^{N}\subset\mathbb{R}^3$ and $Q=\{q_j\}_{j=1}^{M}\subset\mathbb{R}^3$, the squared bidirectional Chamfer Distance is
\begin{equation}
\label{eq:cd_def}
\mathrm{CD}(P,Q) = \frac{1}{N}\sum_{p\in P}\min_{q\in Q}\|p-q\|_2^2 + \frac{1}{M}\sum_{q\in Q}\min_{p\in P}\|q-p\|_2^2.
\end{equation}
\end{definition}

When applied to assemblies, CD operates on a projected point set $\Phi(\mathcal{A})=\bigcup_{i=1}^{K}\mathrm{Sample}(T_i S_i)$, where each part $S_i$ is transformed by $T_i\in SE(3)$. This projection discards part identity, assembly graph structure, contact relationships, kinematic constraints, and physical feasibility.

\subsection{Information Loss under Point-Set Projection}

The fundamental issue with applying CD to assemblies is the information loss incurred by projecting a structured assembly into an unordered point set. The following lemma formalizes this observation.

\begin{lemma}[Irrecoverability of Assembly Structure]
\label{lem:irrecoverable}
Let $f$ be any evaluation function that depends only on point-set projections, i.e., $f(\mathcal{A},\mathcal{A}^{\star})=g(\Phi(\mathcal{A}),\Phi(\mathcal{A}^{\star}))$ for some function $g$. If two assemblies $\mathcal{A}_1,\mathcal{A}_2$ satisfy $\Phi(\mathcal{A}_1)=\Phi(\mathcal{A}_2)$, then $f(\mathcal{A}_1,\mathcal{A}^{\star})=f(\mathcal{A}_2,\mathcal{A}^{\star})$ for any reference $\mathcal{A}^{\star}$.
\end{lemma}

\begin{proof}
By definition, $f(\mathcal{A}_1,\mathcal{A}^{\star})=g(\Phi(\mathcal{A}_1),\Phi(\mathcal{A}^{\star}))=g(\Phi(\mathcal{A}_2),\Phi(\mathcal{A}^{\star}))=f(\mathcal{A}_2,\mathcal{A}^{\star})$.
\end{proof}

Since CD is precisely such a function, it inherits this limitation: any assembly information lost during point-set projection is invisible to CD.

\subsection{CD Lacks Assembly Structure Identifiability}

Building on the irrecoverability result, we show that CD cannot distinguish assemblies that differ in structurally important ways, such as part count, assembly graph topology, or constraint relationships whenever their surface point sets coincide.

\begin{proposition}[Non-Identifiability]
\label{prop:non_ident}
There exist assemblies $\mathcal{A}_1$ and $\mathcal{A}_2$ that differ in part count, assembly graph, or constraint relationships, yet satisfy $\Phi(\mathcal{A}_1)=\Phi(\mathcal{A}_2)$, and hence $\mathrm{CD}(\Phi(\mathcal{A}_1),\Phi(\mathcal{A}^{\star}))=\mathrm{CD}(\Phi(\mathcal{A}_2),\Phi(\mathcal{A}^{\star}))$ for any reference $\mathcal{A}^{\star}$.
\end{proposition}

\begin{proof}
Consider a two-part assembly $\mathcal{A}_1=(\{S_1,S_2\},\{T_1,T_2\},\mathcal{G}_1,\mathcal{C}_1)$. Construct $\mathcal{A}_2$ as a single fused part $S'=T_1 S_1 \cup T_2 S_2$ with identity pose. Although $\mathcal{A}_1$ and $\mathcal{A}_2$ differ in part count, assembly graph, and disassembly relationships, their surface point sets are identical: $\Phi(\mathcal{A}_1)=\Phi(\mathcal{A}_2)$. By Lemma~\ref{lem:irrecoverable}, CD cannot distinguish them.
\end{proof}

This result extends to part identity swaps, left-right mirror errors, and missing constraint relationships---any structural defect that preserves the overall surface geometry remains invisible to CD.

\subsection{CD Lacks Assembly Ranking Consistency}

Beyond identifiability, a reliable assembly metric should rank a structurally superior prediction closer to the reference. We formalize this property and prove that CD violates it: a structurally complete assembly can receive a worse CD score than one missing a functionally critical component.

\begin{definition}[Assembly Ranking Consistency]
A metric $d$ is assembly-ranking-consistent if, whenever $Q_{\mathrm{asm}}(\mathcal{A}_1,\mathcal{A}^{\star})<Q_{\mathrm{asm}}(\mathcal{A}_2,\mathcal{A}^{\star})$ for an ideal assembly quality function $Q_{\mathrm{asm}}$, it follows that $d(\mathcal{A}_1,\mathcal{A}^{\star})<d(\mathcal{A}_2,\mathcal{A}^{\star})$.
\end{definition}

\begin{theorem}[Ranking Reversal]
\label{thm:reversal}
CD is not assembly-ranking-consistent. There exist a reference $P^{\star}$ and two predictions $P_{\mathrm{good}}$, $P_{\mathrm{bad}}$ such that $P_{\mathrm{good}}$ is structurally superior yet $\mathrm{CD}(P_{\mathrm{good}},P^{\star})>\mathrm{CD}(P_{\mathrm{bad}},P^{\star})$.
\end{theorem}

\begin{proof}
Let the reference consist of a large part with $N$ sample points $P_L$ and a small but functionally critical part with $m$ sample points $P_S$, so $P^{\star}=P_L\cup P_S$ with $N\gg m$. Let $\bar{D}^2=\frac{1}{m}\sum_{s\in P_S}\min_{l\in P_L}\|s-l\|_2^2$ denote the average squared nearest-neighbor distance from the small part to the large part.

Construct two predictions:
\begin{itemize}[nosep]
\item $P_{\mathrm{good}}=(P_L+\epsilon\mathbf{u})\cup P_S$: all parts present, large part shifted by $\epsilon$.
\item $P_{\mathrm{bad}}=P_L$: large part perfectly aligned, small critical part missing.
\end{itemize}

Direct computation yields
\[
\mathrm{CD}(P_{\mathrm{good}},P^{\star})\approx\frac{2N\epsilon^2}{N+m}, \qquad \mathrm{CD}(P_{\mathrm{bad}},P^{\star})=\frac{m\bar{D}^2}{N+m}.
\]
Ranking reversal occurs when $2N\epsilon^2>m\bar{D}^2$, i.e.,
\begin{equation}
\label{eq:reversal}
N > \frac{m\bar{D}^2}{2\epsilon^2}.
\end{equation}
Since large structural parts (housings, frames, base plates) typically have $N\gg m$ relative to small connectors (bolts, pins, bearings), this condition is easily satisfied in practice. Thus a structurally complete prediction receives a worse CD score than one missing a functionally critical component.
\end{proof}

\subsection{Surface-Area Bias}

Under uniform surface sampling, the number of points from part $i$ satisfies $N_i\approx N_{\mathrm{total}}\cdot A_i / A_{\mathrm{total}}$. Therefore, global CD implicitly weights each part's error by its surface area fraction:
\[
\mathrm{CD}_{\mathrm{global}} \approx \sum_{i=1}^{K}\frac{A_i}{A_{\mathrm{total}}}\,e_i,
\]
where $e_i$ is the per-part average squared error. Large-surface-area components dominate the score regardless of functional importance, while small but mechanically critical parts (fasteners, bearings, shafts) contribute negligibly.

\end{document}